\documentclass[lettersize,journal]{IEEEtran}
\usepackage{amsmath,amsfonts}
\usepackage{algorithmic}
\usepackage{algorithm}
\usepackage{array}
\usepackage[caption=false,font=normalsize,labelfont=sf,textfont=sf]{subfig}
\usepackage{textcomp}
\usepackage{stfloats}
\usepackage{url}
\usepackage{verbatim}
\usepackage{graphicx}
\usepackage{makecell}
\usepackage{booktabs}
\usepackage{multirow}
\usepackage{amsthm}
\usepackage{enumitem}
\usepackage{marvosym}
\usepackage{booktabs}
\usepackage[numbers,sort&compress]{natbib}
\usepackage[hidelinks]{hyperref}
\usepackage{algorithm}
\usepackage{algorithmic}
 
\makeatletter
\newcommand{\removelatexerror}{\let\@latex@error\@gobble}
\makeatother

\usepackage{multirow}
\usepackage{stfloats}
\usepackage{colortbl}
\usepackage{makecell}
\usepackage{thmtools}
\usepackage{thm-restate}
\usepackage[table]{xcolor}

\newtheorem{theorem}{Theorem}
\newtheorem{assumption}{Assumption}

\makeatletter
\let\supplementmaketitle\maketitle
\let\supplementinternalmaketitle\@maketitle
\makeatother

\begin{document}

\title{Visual-Invariance-Augmented Feature Optimal Alignment for Transferable Adversarial Attacks against Closed-Source MLLMs}

\author{Xiaojun Jia,
        Simeng Qin,
        Yiming Li,
        Jie Liao, 
        Sensen Gao, \\
        Ke Ma,
        Yang Liu,~\IEEEmembership{Senior Member,~IEEE},
        and Xiaochun Cao~\IEEEmembership{Senior Member,~IEEE}%
\IEEEcompsocitemizethanks{
\IEEEcompsocthanksitem Xiaojun Jia, Yang Liu, Yiming Li, and Sensen Gao are with Nanyang Technological University, Singapore (e-mail: jiaxiaojunqaq@gmail.com; yangliu@ntu.edu.sg; liyiming.tech@gmail.com; sensen.gao2002@gmail.com).
\IEEEcompsocthanksitem Simeng Qin is with Northeastern University, Shenyang, Liaoning, China (e-mail: qinsimeng@neuq.edu.cn).
\IEEEcompsocthanksitem Jie Liao is with Chongqing University, Chongqing, China (e-mail: liaojie@cqu.edu.cn).
\IEEEcompsocthanksitem Ke Ma is with the School of Electronic, Electrical and Communication Engineering, University of Chinese Academy of Sciences, Beijing 100049, China (e-mail: make@ucas.ac.cn).
\IEEEcompsocthanksitem Xiaochun Cao is with the School of Cyber Science and Technology, Shenzhen Campus of Sun Yat-sen University, Shenzhen 518107, China (e-mail: caoxiaochun@mail.sysu.edu.cn).
}%
}

\markboth{Manuscript for IEEE Transactions on Pattern Analysis and Machine Intelligence}%
{Jia \MakeLowercase{\textit{et al.}}: Visual-Invariance-Augmented Feature Optimal Alignment}


\maketitle

\begin{abstract}
Multimodal large language models (MLLMs) remain vulnerable to transferable adversarial examples, especially in black-box scenarios where only open-source surrogate models are accessible. Existing targeted transfer attacks mainly align adversarial and target samples using global image-level features, such as the [CLS] embedding of vision-language encoders. However, such coarse alignment insufficiently exploits patch-level visual structures, leading to limited transferability across heterogeneous closed-source MLLMs. To address this issue, we propose a visual-invariance-augmented feature optimal alignment attack with adaptive unbalanced transport, called IAU-FOA, for improving targeted transferability against closed-source MLLMs. IAU-FOA aligns adversarial and target samples from both global and local perspectives: a cosine-based objective narrows their global semantic gap, while patch tokens are clustered into compact local patterns and matched through optimal transport for fine-grained feature alignment. Balanced optimal transport enforces fixed marginal masses even for local clusters without reliable counterparts, potentially introducing misleading alignment gradients. We therefore introduce confidence-adaptive unbalanced transport to relax these constraints for weakly matched clusters, aiming to reduce unreliable local alignment and improve adversarial transferability. Moreover, we study the effect of input transformations on adversarial transferability and propose a visual-invariance augmentation strategy that applies bidirectional pixel-intensity rescaling and per-channel white-balance adjustment to simulate exposure, contrast, illumination, and color-temperature variations, thereby encouraging adversarial perturbations to generalize across different visual encoders. Extensive experiments on open-source and closed-source MLLMs show that IAU-FOA consistently outperforms state-of-the-art transferable attack methods. The code is released at \href{https://github.com/jiaxiaojunQAQ/IAU-FOA}{https://github.com/jiaxiaojunQAQ/IAU-FOA}.
\end{abstract}

\section{Introduction}
\IEEEPARstart{M}ultimodal large language models (MLLMs)~\citep{yin2024survey,fu2026mme} have become a representative framework for connecting visual perception with language reasoning~\citep{huang2025vision,kuang2025natural}. In most current systems, an input image is first encoded into visual tokens or embeddings by a vision encoder, and these visual representations are then projected into the language model for multimodal understanding and generation~\citep{li2023blip, liu2023visual, zhu2024minigpt, lu2024deepseek, wang2024qwen2}. In particular, many MLLMs adopt pre-trained vision encoders, such as Contrastive Language-Image Pre-training (CLIP)~\citep{radford2021learning}, to extract image features and align them with language embeddings. This modular design has enabled strong performance on a wide range of vision-language applications. Beyond advances in open-source MLLMs, commercial closed-source MLLMs such as GPT-5.5, Claude-Opus-4.8, and Gemini-3.5-Flash have also been widely deployed in practical multimodal services. However, since MLLMs heavily rely on vision encoders to perceive visual inputs, they may also inherit the adversarial vulnerabilities of these visual front ends~\citep{yuan2026frequency,lu2026make,wang2026adversarial}. Recent studies~\citep{dong2023robust,gao2024boosting,liu2024safety,mei2025veattack} have shown that MLLMs can be vulnerable to adversarial examples~\citep{DBLP:journals/corr/GoodfellowSS14}, which are typically generated by adding small and visually imperceptible perturbations to clean images, such that the perturbed images remain natural to humans but change the model's responses. Such vulnerability raises security and safety concerns for real-world foundation-model services, where manipulated visual inputs may mislead downstream multimodal reasoning. According to the attack objective, adversarial attacks on MLLMs can be broadly divided into untargeted and targeted attacks. Untargeted attacks aim to induce incorrect responses, whereas targeted attacks attempt to drive the model toward attacker-specified semantics or outputs~\cite{yuan2019adversarial,jiang2025survey}.

\begin{figure}[t]
\begin{center}
 \includegraphics[width=0.9\linewidth]{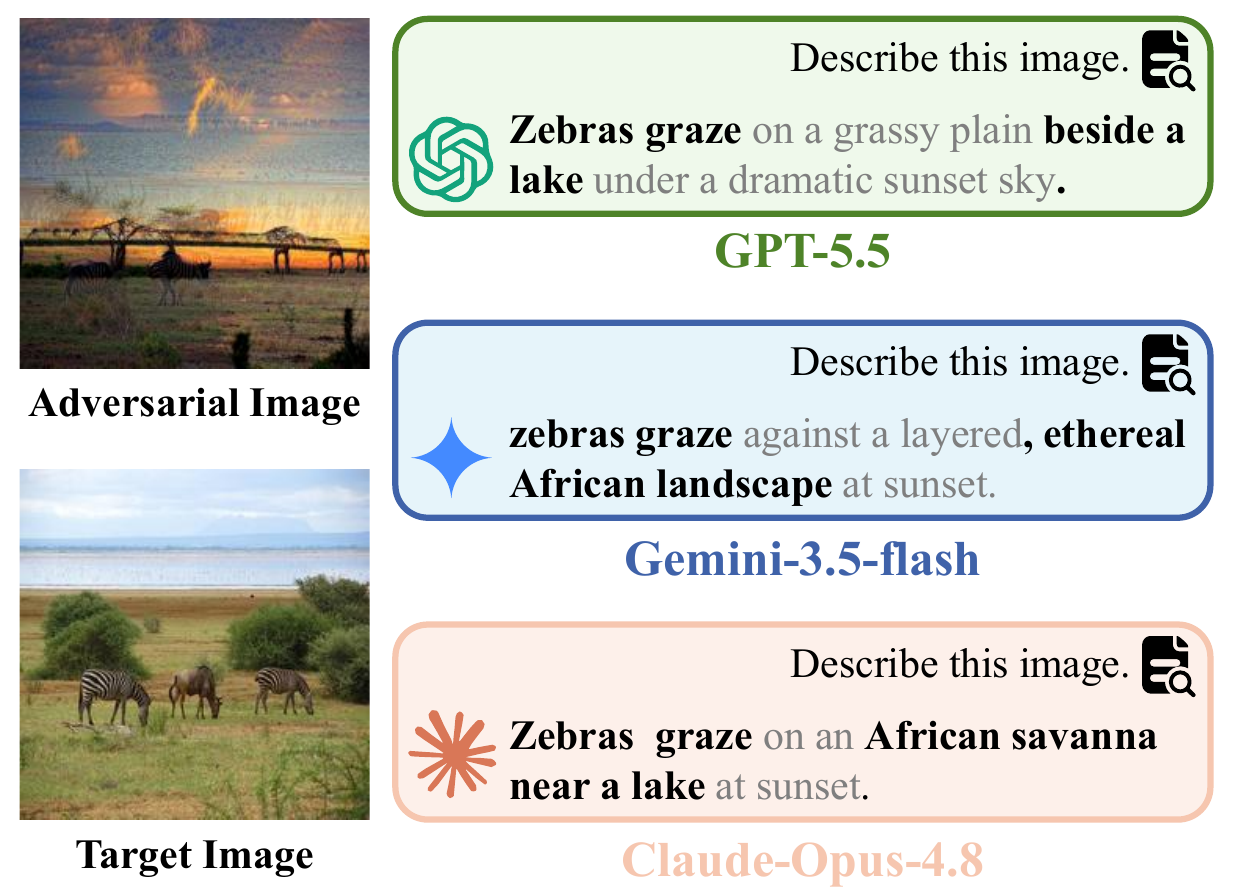}
\end{center}
\vspace{-5mm}
\caption{Examples of targeted transfer attacks against commercial MLLMs. The adversarial images remain visually close to the source images, but induce target-related responses to the prompt ``Describe this image''.}
        \label{fig:home}
\end{figure}

As illustrated in Fig.~\ref{fig:home}, targeted adversarial examples can preserve the visual appearance of the source image while steering commercial MLLMs toward target-related descriptions. In closed-source scenarios, the gradients and parameters of target MLLMs are inaccessible, making direct optimization infeasible. A practical alternative is to craft adversarial examples on open-source surrogate models and then transfer them to the target systems. This property, known as adversarial transferability~\citep{gu2023survey}, measures whether adversarial examples generated for one model can remain effective on other unseen models. Transferability is important for both untargeted and targeted attacks, but it is particularly difficult to achieve in the targeted black-box setting, where the perturbation must consistently convey the target semantics across heterogeneous MLLMs rather than merely causing prediction errors~\citep{byun2022improving,zhu2023boosting}. Therefore, an effective targeted attack should not simply exploit surrogate-specific weaknesses, but should learn perturbations that preserve transferable visual semantics across different visual encoders and preprocessing pipelines.

Previous targeted transfer attack methods~\citep{li2026frustratingly,li2026multi,wang2026adversarial} typically generate adversarial examples by minimizing a contrastive loss between the global features of the adversarial image and the target image, where the global representation is often given by the \texttt{[CLS]} token of an open-source image encoder such as CLIP. By pulling the adversarial image toward the target image in the surrogate feature space, these methods can improve targeted transferability to some extent. However, their adversarial transferability and attack performance against closed-source MLLMs remain limited, mainly due to two limitations. First, most existing methods rely primarily on global image-level features. Although global feature alignment can guide adversarial examples toward the target semantics, it compresses the visual content into a single representation and overlooks rich patch-level structures. As a result, the generated perturbations may capture coarse semantic similarity while failing to sufficiently exploit fine-grained visual patterns that are important for transfer across heterogeneous MLLMs. Even when local features are aligned, exact marginal constraints can force matches between clusters without reliable counterparts. These matches may introduce gradients that favor incidental similarities in a surrogate encoder rather than correspondences shared across models, motivating confidence-adaptive local transport. Second, the choice of input transformations for targeted transfer attacks against MLLMs warrants further investigation. Existing approaches include spatial transformations such as resizing and padding in DI~\citep{xie2019improving}, as well as pixel-intensity scaling in SIM~\citep{lin2019nesterov}. Building on intensity-based augmentation, we investigate the combination of bidirectional intensity scaling and channel-wise gain adjustment to diversify exposure and illumination-color conditions during targeted feature alignment.

To address these issues, we propose \textbf{IAU-FOA}, a visual-invariance-augmented feature optimal alignment attack with an adaptive unbalanced transport method. Given a source image and a target image, IAU-FOA generates an adversarial perturbation by aligning their representations in surrogate vision-language encoders from both global and local perspectives. At the global level, IAU-FOA adopts a coarse-grained feature alignment objective based on cosine similarity, which pulls the image-level representation of the adversarial example toward that of the target sample. Since the global representation is usually encoded by the \texttt{[CLS]} token, this objective provides a stable semantic direction for targeted attacks and encourages the adversarial image to approach the target semantics in the surrogate feature space.

At the local level, IAU-FOA further exploits the patch-token representations produced by Transformer-based vision encoders. While the \texttt{[CLS]} token summarizes global information, the remaining visual tokens contain local patch-level details. Directly using all local tokens, however, can introduce redundancy and noisy correspondences. Therefore, IAU-FOA first clusters patch tokens into compact and discriminative local visual patterns, where each cluster center summarizes the characteristics of a group of related patches. It then formulates the alignment between adversarial and target local patterns as an optimal transport problem, enabling fine-grained feature matching beyond a single global embedding.

To make local alignment more reliable, IAU-FOA further introduces an adaptive unbalanced transport mechanism. Balanced OT requires each cluster to transport its prescribed mass, even when reliable counterparts are absent. The proposed scheme instead adapts marginal regularization to cluster-level matching confidence. High-confidence clusters receive stronger penalties on marginal deviations, while low-confidence clusters receive weaker penalties. This flexibility aims to reduce the interference of unreliable local matches in adversarial optimization and thereby improve adversarial transferability. In addition, when multiple visual encoders are used as surrogate models, IAU-FOA computes an alignment similarity score for each encoder and dynamically adjusts its weight using the ratio between the current and previously recorded scores. For positive scores, larger relative increases receive larger weights, allowing the ensemble to respond to changes in alignment across surrogate encoders.

Finally, IAU-FOA incorporates a visual-invariance augmentation (VIA) strategy during adversarial optimization. Building on existing intensity-scaling augmentation, our strategy combines bidirectional pixel-intensity rescaling with channel-wise gain adjustment to approximate exposure and white-balance variations during targeted alignment. By optimizing adversarial perturbations under these transformed visual conditions, IAU-FOA encourages the perturbations to rely less on surrogate-specific visual statistics and to remain effective across heterogeneous closed-source MLLMs. Our main contributions are summarized as follows:
\begin{itemize}
    \item We propose \textbf{IAU-FOA}, a targeted black-box attack framework for MLLMs that drives adversarial images toward target semantics through a dual-granularity feature matching objective. Specifically, a cosine-based loss provides coarse semantic guidance from global embeddings, while clustered patch tokens are matched through an optimal-transport formulation to capture discriminative local structures.
    \item We propose an adaptive unbalanced transport mechanism to improve the reliability of fine-grained matching. It relaxes marginal constraints for weakly matched clusters, aiming to reduce misleading alignment gradients arising from unreliable local correspondences and improve targeted transferability.
    \item We propose a dynamic weighting strategy for ensemble surrogate optimization. The strategy adjusts each encoder's contribution according to the relative change in its alignment similarity score, providing a score-based alternative to fixed ensemble weights for targeted transfer.
    \item We propose a visual-invariance augmentation strategy for transferable adversarial optimization. Through bidirectional intensity scaling and channel-wise white-balance adjustment, the attack is optimized under diverse photometric conditions and becomes less dependent on surrogate-specific visual statistics.
    \item Extensive experiments on open-source and closed-source MLLMs demonstrate that IAU-FOA consistently outperforms state-of-the-art transferable attack methods.
\end{itemize}

This paper is an extension of our preliminary conference version, FOA~\citep{jia2026adversarial}, which forms the foundation of the proposed framework. Compared with the original balanced feature optimal alignment framework, we extend this work in four aspects: 1) We introduce visual-invariance-augmented optimization in \textbf{Section~\ref{sec:via}} to improve robustness against photometric changes and encoder-dependent visual statistics. 2) We develop an adaptive unbalanced transport method for local feature alignment in \textbf{Section~\ref{sec:adaptive_uot}}, which replaces exact marginal preservation with confidence-dependent penalties on marginal deviations. We add a theoretical analysis of the proposed method in \textbf{Section~\ref{sec:transfer_analysis}}. 3) We conduct more experiments and analyses, including comparisons with state-of-the-art methods, ablation studies, and evaluation under adversarial defenses. We include additional state-of-the-art adversarial attack methods for comparison in \textbf{Section~\ref{sec:comparative_results}}. We add ablation studies of the proposed components in \textbf{Section~\ref{sec:ablation_study}} and evaluate the proposed method under adversarial defenses in \textbf{Section~\ref{sec:performance_analysis}}. 4) We have thoroughly revised the abstract, introduction, method, experiments, and conclusion sections to offer a more detailed overview of our motivation and approach. We have updated all figures and tables to enhance clarity and presentation.

\section{Related Work}

\subsection{Multimodal Large Language Models}
Large language models (LLMs)~\citep{zhao2023survey} have achieved impressive success across a wide range of Natural Language Processing (NLP) tasks, demonstrating strong capabilities in language understanding~\citep{karanikolas2023large}, instruction following~\citep{harada2025instructions}, and complex reasoning~\citep{ma2026general,guo2025r}. Building upon these advances, multimodal large language models (MLLMs) extend LLMs to vision-language scenarios by integrating visual perception modules with LLM backbones. Specifically, MLLMs align visual representations with the semantic space of LLMs, enabling the language model to interpret visual inputs and generate grounded textual responses~\citep{alayrac2022flamingo,su2023pandagpt,liu2024improved}. Representative systems usually connect pre-trained vision encoders and LLMs through lightweight alignment modules, such as the Q-Former in BLIP-2~\citep{li2023blip} and the MLP projector in LLaVA~\citep{liu2023visual}. Some open-source models further enhance fine-grained visual perception through versatile visual instruction tuning and high-resolution image processing, as demonstrated by Qwen2-VL~\citep{wang2024qwen2} and InternVL~\citep{chen2024far}. Meanwhile, commercial MLLMs, such as GPT-5.5 and Claude-Opus-4.8, have exhibited powerful multimodal understanding, reasoning, and instruction-following capabilities, representing the leading performance of current vision-language systems. However, their closed-source nature, together with substantial differences in model architectures, vision encoders, alignment modules, and training recipes, makes their internal representation spaces difficult to access or approximate, posing challenges for black-box attacks against modern MLLMs.

\subsection{Adversarial Attacks}
Adversarial attacks aim to mislead deep neural networks by adding carefully crafted perturbations to clean inputs~\citep{szegedy2013intriguing,DBLP:journals/corr/GoodfellowSS14}. According to the attacker's access to the victim model, existing adversarial attack methods are commonly categorized into white-box attacks, where model parameters or gradients are available, and black-box attacks, where the attacker has no access to internal model parameters, gradients, or feature representations \cite{madry2017towards,zhang2025adversarial}. Depending on the attack objective, adversarial attacks can also be divided into untargeted attacks and targeted attacks. Untargeted attacks only require the model to produce incorrect outputs, whereas targeted attacks aim to induce specific target semantics or responses. Compared with conventional adversarial attacks on image classification models, attacking MLLMs is substantially more challenging~\citep{jiang2025survey,jain2026adversarial}. Classification models usually produce predictions over a fixed label space, allowing attacks to optimize well-defined objectives. In contrast, MLLMs generate open-ended textual responses conditioned on both visual and textual inputs, where the attack goal is often expressed at the semantic or instruction-following level rather than as a discrete label. For closed-source MLLMs, the challenge becomes more pronounced because their parameters, gradients, and intermediate feature representations are unavailable to attackers. A potential solution is to exploit adversarial transferability~\citep{papernot2016transferability,liu2016delving}, where adversarial examples generated on accessible surrogate models remain effective on unseen target models. This property enables transfer-based attacks to probe closed-source MLLMs without directly accessing their internal states, making it a practical and promising route for targeted black-box attacks~\citep{xie2019improving,lin2019nesterov}.

\subsection{Transfer-Based Targeted Attacks for MLLMs}
Recent transfer-based attacks on vision-language models typically optimize adversarial perturbations on accessible surrogate models, such as CLIP, and transfer them to closed-source MLLMs. Specifically, Guo \textit{et al.}~\citep{guo2024efficient} propose AdvDiffVLM, which leverages diffusion models to generate natural and transferable targeted adversarial examples for vision-language models. Zhang \textit{et al.}~\citep{zhang2025anyattack} propose AnyAttack, a self-supervised framework for large-scale generation of label-free targeted adversarial examples, improving the scalability of targeted transfer attacks. Xie \textit{et al.}~\citep{xie2025chain} propose Chain of Attack (COA), a transfer-based black-box attack method that iteratively generates adversarial examples through multimodal semantic updates between visual and textual modalities.
To better exploit visual details for closed-source MLLM attacks, Li \textit{et al.}~\citep{li2026frustratingly} propose M-Attack, which constructs a simple yet effective black-box attack baseline by enhancing target-related visual patterns. Zhao \textit{et al.}~\citep{zhao2026pushing} further propose M-Attack-v2, which emphasizes fine-grained detail targeting to improve attack effectiveness against strong black-box MLLMs. Li \textit{et al.}~\citep{li2026multi} propose MPCAttack, which collaboratively aggregates multi-paradigm vision-language features to reduce surrogate representation bias and strengthen transfer-based attacks against MLLMs. Despite these advances, existing targeted transfer attacks still show limited effectiveness against advanced closed-source MLLMs, partly because they often rely on surrogate-specific global representations or do not sufficiently account for the reliability of local correspondences. In addition, fine-grained adaptive alignment and visual-invariance optimization remain insufficiently explored.

\begin{figure*}[t]
\begin{center}
 \includegraphics[width=0.95\linewidth]{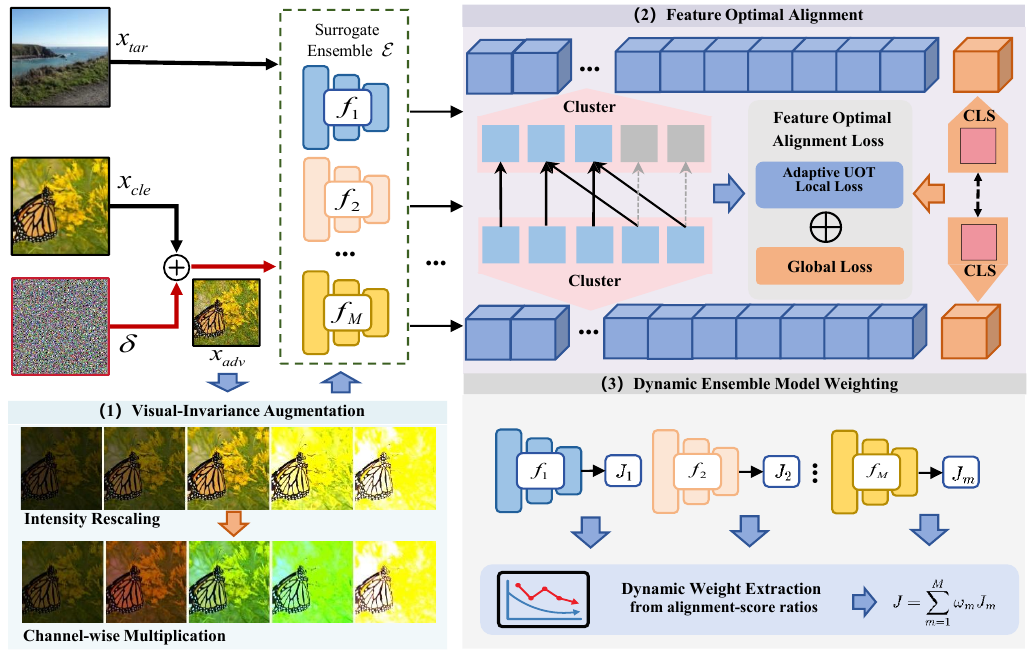}
\end{center}
\vspace{-6mm}
\caption{Overview of IAU-FOA: visual-invariance augmentation, global feature alignment, adaptive unbalanced local transport, and dynamic surrogate weighting.}
        \label{fig:framework}
\end{figure*}

\section{Method}

Adversarial perturbations learned from a single surrogate may capture non-transferable representation biases, which can be alleviated by optimizing over multiple surrogate image encoders to emphasize shared feature directions across models~\citep{rahmatullaev2025universal,yang2025use,wang2026adversarial}. Following this perspective, IAU-FOA adopts a dynamic ensemble framework for targeted transfer attacks against closed-source MLLMs. The proposed framework integrates dual-granularity feature optimal alignment, adaptive unbalanced local transport, dynamic encoder weighting, and visual-invariance augmentation (VIA). By jointly aligning target semantics across multiple surrogate visual encoders and adapting their weights to relative changes in alignment similarity, IAU-FOA encourages adversarial perturbations to capture transferable visual cues shared by foundation models.

\subsection{Problem Formulation}
We study targeted transfer-based attacks against closed-source MLLMs using an ensemble of accessible surrogate visual encoders. Let
\begin{equation}
    \mathcal{F}
    =
    \{f_m\}_{m=1}^{M}
\end{equation}
denote $M$ surrogate visual encoders, including vision-language and vision-only models. For an input image $\mathbf{x}\in\mathcal{X}$, where $\mathcal{X}$ denotes the valid image domain, the global and patch-level representations supplied by the $m$-th feature extractor are denoted as
\begin{equation}
    f_m(\mathbf{x})
    =
    [\mathbf{z}_{m}^{g}(\mathbf{x});
    \mathbf{Z}_{m}^{p}(\mathbf{x})]
    \in
    \mathbb{R}^{(N_m+1)\times d_m},
\end{equation}
where $\mathbf{z}_{m}^{g}(\mathbf{x})\in\mathbb{R}^{d_m}$ is the global representation supplied by the feature extractor (e.g., a [CLS] embedding), and $\mathbf{Z}_{m}^{p}(\mathbf{x})=[\mathbf{z}_{m,1}^{p}(\mathbf{x}),\ldots,\mathbf{z}_{m,N_m}^{p}(\mathbf{x})]^{\top}\in\mathbb{R}^{N_m\times d_m}$ denotes the patch-token features. Given a source image $\mathbf{x}_s$ and a target image $\mathbf{x}_{tar}$, the attacker generates an adversarial example
\begin{equation}
    \mathbf{x}_{adv}
    =
    \mathbf{x}_s+\boldsymbol{\delta},
    \quad
    \|\boldsymbol{\delta}\|_{\infty}\leq \epsilon ,
\end{equation}
where $\boldsymbol{\delta}$ is the bounded adversarial perturbation. The goal is to make the adversarial image near the target image in the feature spaces of all surrogate encoders:
\begin{equation}
    \min_{\|\boldsymbol{\delta}\|_{\infty}\leq\epsilon}
    \sum_{m=1}^{M}
    \mathcal{D}_{m}
    \left(
    f_m(\mathbf{x}_{adv}),
    f_m(\mathbf{x}_{tar})
    \right),
    \label{eq:preliminary_objective}
\end{equation}
where $\mathcal{D}_{m}(\cdot,\cdot)$ denotes the feature-alignment objective on the $m$-th encoder. Since the victim MLLM is closed-source, its parameters and gradients are unavailable, and the perturbation optimized on $\mathcal{F}$ is directly transferred to unseen MLLMs. IAU-FOA optimizes the adversarial perturbation by reducing the feature-alignment loss between the adversarial image and the target image from both global and local perspectives.

\subsection{Overview of the Proposed Framework}
The proposed IAU-FOA consists of four components: visual-invariance augmentation, dual-granularity feature alignment, adaptive unbalanced optimal transport, and dynamic ensemble weighting. The current target view is passed through each surrogate encoder to obtain global and clustered local target features. We use $\mathbf{x}_{tar}$ below to denote the target view supplied to the alignment objective, and suppress spatial preprocessing in the notation. During optimization, multiple photometrically transformed copies of the adversarial image are generated. For each copy, IAU-FOA computes a weighted ensemble objective over surrogate encoders and averages the gradients from all valid copies. The loss minimized at each iteration can be written as
\begin{equation}
\begin{aligned}
    \min_{\|\boldsymbol{\delta}\|_{\infty}\leq \epsilon}
    \frac{1}{|\mathcal{B}|}
    \sum_{\mathcal{T}\in\mathcal{B}}
    \sum_{m=1}^{M}
    \omega_m^{(k(\mathcal{T}))}
    \Big(
    &\mathcal{L}_{g}^{m}(\mathcal{T}(\mathbf{x}_{adv}),\mathbf{x}_{tar})\\
    &+
    \lambda_l
    \mathcal{L}_{l}^{m}(\mathcal{T}(\mathbf{x}_{adv}),\mathbf{x}_{tar})
    \Big),
\end{aligned}
\label{eq:overall_objective}
\end{equation}
where $\mathcal{B}$ is the set of visual-invariance transformations sampled at the current iteration, $k(\mathcal{T})$ indexes the objective evaluation for the transformed copy $\mathcal{T}(\mathbf{x}_{adv})$, $\omega_m^{(k(\mathcal{T}))}$ is the corresponding dynamic weight of the $m$-th surrogate encoder, and $\lambda_l$ controls the contribution of local alignment. The weights are updated at each objective evaluation and held fixed during backpropagation. Eq.~\eqref{eq:overall_objective} summarizes the optimization used to update the perturbation. The outer average makes the perturbation effective under different photometric views $\mathcal{T}(\mathbf{x}_{adv})$, which reduces overfitting to the visual statistics of a single input rendering. For each transformed copy, the inner summation aggregates the attack loss over all surrogate encoders. Within each surrogate, $\mathcal{L}_{g}^{m}$ measures the global feature alignment loss to the target image, while $\mathcal{L}_{l}^{m}$ measures the transported local alignment loss between clustered patch-level structures. Thus, minimizing Eq.~\eqref{eq:overall_objective} drives $\mathbf{x}_{adv}$ toward the target image simultaneously at the semantic level and at the local visual-pattern level. The weight $\omega_m^{(k(\mathcal{T}))}$ is computed from relative changes in alignment similarity scores, as specified in Sec.~\ref{sec:dynamic_weighting}.

\subsection{Visual-Invariance Augmentation}
\label{sec:via}
Real imaging conditions and closed-source MLLM visual preprocessing pipelines often introduce photometric variations, such as changes in exposure, brightness, illumination color, and white balance. Accurately modeling the full physical image formation process is difficult for transferable optimization. Instead, we use two simple, differentiable, and low-cost operations, namely global intensity scaling and channel-wise gain adjustment, to cover these common visual variations.  $\mathcal{T}(\mathbf{x}_{adv})$ denotes a transformed view of the adversarial image. To keep the notation compact, we use $\mathcal{T}$ as a generic augmentation operator in the overall objective, and instantiate it here with photometric transformations. Specifically, each transformation combines an intensity scale $s$ and a channel-wise gain vector $\boldsymbol{\beta}=[\beta_R,\beta_G,\beta_B]^{\top}$. We enumerate the intensity scale $s$ over
\begin{equation}
    \mathcal{S}
    =
    \{2^{k}\mid k=-h,\ldots,h\},
\end{equation}
where $h$ is a nonnegative integer controlling the range of intensity scaling, and $k$ indexes the $2h+1$ scales. Scaling darkens ($s<1$), brightens ($s>1$), or preserves ($s=1$) image intensity. By default, $h=2$, giving five bidirectional intensity scales $\mathcal{S}=\{0.25,0.5,1,2,4\}$. Given $s\in\mathcal{S}$ and $\boldsymbol{\beta}\in\mathbb{R}^{3}$, the concrete operator used in Eq.~\eqref{eq:overall_objective} is
\begin{equation}
    \mathcal{T}(\mathbf{x};s,\boldsymbol{\beta})
    =
    \boldsymbol{\beta}\odot
    (s\mathbf{x}),
    \label{eq:augmentation}
\end{equation}
where $\odot$ denotes channel-wise multiplication. For an RGB image $\mathbf{x}\in\mathbb{R}^{H\times W\times 3}$, this operation is defined as
\begin{equation}
    [\mathcal{T}(\mathbf{x};s,\boldsymbol{\beta})]_{u,v,c}
    =
    \beta_c \cdot s \cdot \mathbf{x}_{u,v,c},
    \quad
    c\in\{R,G,B\}.
\end{equation}
Here, $H$ and $W$ denote the image height and width, $u$ and $v$ index the spatial pixel location, and $c$ indexes the RGB channel. The global scale $s$ is shared across all pixels and channels, while $\beta_c$ is the gain applied to all pixels in channel $c$. For each transformed copy, the channel gains are independently sampled as
\begin{equation}
    \beta_c
    \sim
    \mathcal{U}(1-s_{\beta},1+s_{\beta}),
    \quad
    c\in\{R,G,B\}.
\end{equation}
Here, $\mathcal{U}(a,b)$ denotes the uniform distribution over $[a,b]$, and $s_{\beta}$ controls the amplitude of channel-wise gain variations around one. We set $s_{\beta}=0.5$ by default, giving a gain range of $[0.5,1.5]$; $s_{\beta}=0$ disables channel-wise gain adjustment.

The global scale $s$ approximates changes in brightness and exposure, while the channel gains $\boldsymbol{\beta}$ approximate white-balance and illumination-color shifts. The transformation set $\mathcal{B}$ in Eq.~\eqref{eq:overall_objective} is constructed by enumerating $s\in\mathcal{S}$ and sampling a channel-gain vector $\boldsymbol{\beta}$ for each scale. These transformations simulate exposure, illumination, and color-temperature changes. In each iteration, gradients from all transformed copies are averaged before updating the perturbation, which encourages the adversarial example to remain effective under heterogeneous visual preprocessing pipelines. Each transformed copy is then passed to the surrogate encoders to generate adversarial examples. 

\subsection{Global Feature Alignment}
The global branch uses the encoder-specific global feature to provide coarse-grained semantic guidance. For the $m$-th encoder, the global feature is first normalized as
\begin{equation}
    \hat{\mathbf{z}}_{m}^{g}(\mathbf{x})
    =
    \frac{\mathbf{z}_{m}^{g}(\mathbf{x})}
    {\|\mathbf{z}_{m}^{g}(\mathbf{x})\|_2}.
\end{equation}
Let $\mathbf{X}_{m}=\hat{\mathbf{z}}_{m}^{g}(\mathbf{x}_{adv})$ and $\mathbf{Y}_{m}=\hat{\mathbf{z}}_{m}^{g}(\mathbf{x}_{tar})$ denote the normalized global features of the adversarial and target images, respectively. The global alignment loss is defined by the cosine distance between these two features:
\begin{equation}
\begin{aligned}
    \mathcal{L}_{g}^{m}(\mathbf{x}_{adv},\mathbf{x}_{tar})
    &=
    1-\cos(\mathbf{X}_{m},\mathbf{Y}_{m}) \\
    &=
    1-
    \frac{
    \langle\mathbf{X}_{m},\mathbf{Y}_{m}\rangle
    }{
    \|\mathbf{X}_{m}\|_2\|\mathbf{Y}_{m}\|_2
    },
\end{aligned}
    \label{eq:global_score}
\end{equation}
where $\langle\mathbf{X}_{m},\mathbf{Y}_{m}\rangle$ denotes the inner product between the two feature vectors, and $\|\cdot\|_2$ denotes the $\ell_2$ norm, i.e., the Euclidean length of a vector. Minimizing Eq.~\eqref{eq:global_score} increases the semantic similarity between the adversarial image and the target image in the global feature representation space. 

\subsection{Local Clustering Feature Alignment}
Patch tokens retain spatially distributed visual details beyond a single global representation. However, all patch tokens also include substantial redundant image information, and directly matching them may introduce considerable noisy correspondences, degrading the quality of the generated adversarial examples. To reduce redundancy while preserving discriminative local information, we cluster patch-token features into compact local visual patterns before feature alignment.

For the $m$-th encoder, the fine-grained local features of the adversarial and target images are extracted from the patch-token feature matrices:
\begin{equation}
\begin{aligned}
    \mathbf{X}_{m}^{loc}
    &=
    \mathbf{Z}_{m}^{p}(\mathbf{x}_{adv})
    \in
    \mathbb{R}^{N_m\times d_m}, \\
    \mathbf{Y}_{m}^{loc}
    &=
    \mathbf{Z}_{m}^{p}(\mathbf{x}_{tar})
    \in
    \mathbb{R}^{N_m\times d_m},
\end{aligned}
\end{equation}
where each row corresponds to the feature of one image patch. We then apply clustering to the local features of the adversarial and target images, respectively. Here, $\operatorname{KMeans}(\mathbf{A},K)$ denotes the standard K-means clustering operator applied to the row vectors of $\mathbf{A}$, which partitions them into $K$ clusters by minimizing the within-cluster squared Euclidean distance and returns the cluster centroids:
\begin{equation}
\begin{aligned}
    \mathbf{X}_{m}^{clu}
    &=
    \operatorname{KMeans}\left(\mathbf{X}_{m}^{loc},K\right), \\
    \mathbf{Y}_{m}^{clu}
    &=
    \operatorname{KMeans}\left(\mathbf{Y}_{m}^{loc},K\right),
\end{aligned}
\end{equation}
where $\mathbf{X}_{m}^{clu}=[\mathbf{x}_{m,1}^{clu},\ldots,\mathbf{x}_{m,K}^{clu}]^{\top}\in\mathbb{R}^{K\times d_m}$ and $\mathbf{Y}_{m}^{clu}=[\mathbf{y}_{m,1}^{clu},\ldots,\mathbf{y}_{m,K}^{clu}]^{\top}\in\mathbb{R}^{K\times d_m}$ are the adversarial and target cluster centers. We assign uniform reference mass to each cluster center on both sides. The mass vectors $\mathbf{a}_{m}^{adv}$ and $\mathbf{a}_{m}^{t}$ are defined independently of cluster cardinalities as
\begin{equation}
\begin{aligned}
    a_{m,k}^{adv}
    =
    \frac{1}{K},
    \qquad
    a_{m,k}^{t}
    =
    \frac{1}{K},
    \quad
    k=1,\ldots,K .
\end{aligned}
    \label{eq:cluster_mass}
\end{equation}
Both mass vectors are normalized to sum to one. Following the theory of optimal transport~\citep{villani2009optimal}, we model fine-grained local alignment as a transportation problem between two sets of clustered local patterns. For the $m$-th encoder, the adversarial local distribution is represented by $\{(\mathbf{x}_{m,i}^{clu},a_{m,i}^{adv})\}_{i=1}^{K}$, while the target local distribution is represented by $\{(\mathbf{y}_{m,j}^{clu},a_{m,j}^{t})\}_{j=1}^{K}$. The goal is to find a transport plan that moves the mass of adversarial local patterns to target local patterns with minimum matching cost. Let $\mathbf{P}_{m}(i,j)$ denote the mass transported from the $i$-th adversarial cluster to the $j$-th target cluster, and let $c_{ij}=c(\mathbf{x}_{m,i}^{clu},\mathbf{y}_{m,j}^{clu})$ be their matching cost. The balanced OT form of local alignment is
\begin{equation}
\begin{aligned}
    \min_{\mathbf{P}_{m}}\quad
    &\sum_{i=1}^{K}\sum_{j=1}^{K}
    \mathbf{P}_{m}(i,j)c_{ij},\\
    \mathrm{s.t.}\quad
    &\sum_{j=1}^{K}\mathbf{P}_{m}(i,j)=a_{m,i}^{adv},
    \quad \forall i,\\
    &\sum_{i=1}^{K}\mathbf{P}_{m}(i,j)=a_{m,j}^{t},
    \quad \forall j,\\
    &\mathbf{P}_{m}(i,j)\ge 0,
    \quad \forall i,j .
\end{aligned}
    \label{eq:balanced_ot}
\end{equation}
The marginal constraints in Eq.~\eqref{eq:balanced_ot} preserve the uniform reference masses on both sides. However, strict balanced constraints require prescribed masses even for clusters without reliable counterparts, potentially introducing misleading local alignment gradients. This motivates the adaptive unbalanced formulation introduced below. Specifically, the ground cost is defined by the cosine distance between local cluster centers. We therefore first construct the pairwise similarity matrix
\begin{equation}
    \mathbf{S}_{m}(i,j)
    =
    \cos\left(
    \mathbf{x}_{m,i}^{clu},
    \mathbf{y}_{m,j}^{clu}
    \right).
    \label{eq:similarity_matrix}
\end{equation}
Here, $\mathbf{S}_{m}(i,j)$ measures the cosine similarity between the $i$-th adversarial cluster and the $j$-th target cluster, and the corresponding transport cost is set to $c_{ij}=1-\mathbf{S}_{m}(i,j)$. Minimizing this objective favors low-cost correspondences subject to the prescribed marginals.

\subsection{Adaptive Unbalanced Optimal Transport}
\label{sec:adaptive_uot}
The balanced formulation in Eq.~\eqref{eq:balanced_ot} provides a basic local alignment objective in which every cluster mass is exactly transported according to fixed marginals. With entropic regularization, the balanced OT problem can be solved using Sinkhorn iterations. However, exact marginal preservation can be too restrictive for targeted adversarial optimization: background clusters, ambiguous regions, or surrogate-specific patch patterns may be forced to match target clusters even when no reliable correspondence exists. Such forced matches may steer perturbations toward incidental similarities specific to the surrogate, limiting their transfer to other encoders. We therefore use an adaptive unbalanced Sinkhorn solver with weaker marginal regularization for low-confidence clusters and stronger regularization for high-confidence clusters. This design aims to reduce unreliable alignment by allowing transported masses to adapt to the available correspondences.

Starting from the cosine-distance cost $c_{ij}=1-\mathbf{S}_{m}(i,j)$ defined above, the entropic transport kernel is
\begin{equation}
    \mathbf{K}_{m}(i,j)
    =
    \exp\left(
    -\frac{1-\mathbf{S}_{m}(i,j)}{\varepsilon_{ot}}
    \right),
    \label{eq:transport_kernel}
\end{equation}
where $\varepsilon_{ot}>0$ is the entropy regularization coefficient, also serving as the temperature of the Sinkhorn kernel, and is set to $0.1$ in all experiments. For each adversarial-side cluster on the adaptively relaxed side, the confidence is estimated by its best target counterpart:
\begin{equation}
    r_i
    =
    \max\left\{0,\max_{j}\mathbf{S}_{m}(i,j)\right\}.
    \label{eq:confidence}
\end{equation}
We normalize the raw confidence with a stabilized mean:
\begin{equation}
    \hat{r}_i
    =
    \frac{r_i}
    {\frac{1}{K}\sum_{k=1}^{K}r_k+\eta}.
    \label{eq:normalized_confidence}
\end{equation}
Here, $\eta=10^{-8}$. When the mean raw confidence is well above $\eta$, the normalized confidences have a mean close to one, modulating marginal regularization around its base level. If all raw confidences are zero, every $\hat r_i$ is zero; with $\gamma=1$, this gives $\rho_i^a=\phi_i^a=0$ while the target-side penalty remains unchanged.
We use $\rho>0$ as a hyperparameter controlling the KL penalty on marginal deviations, with $\rho=0.2$ by default. Accordingly, the base unbalanced relaxation coefficient is
\begin{equation}
    \phi=\frac{\rho}{\rho+\varepsilon_{ot}},
\end{equation}
where a larger $\rho$ gives a larger $\phi$ and enforces marginal preservation more strongly. The default settings give $\phi=2/3$. With mean-normalized adaptive confidence, the adversarial-side relaxation coefficient becomes
\begin{equation}
\begin{aligned}
    \rho_i^{a}
    &=
    \rho\left[(1-\gamma)+\gamma\hat{r}_i\right],\\
    \phi_i^{a}
    &=
    \frac{
    \rho\left[(1-\gamma)+\gamma\hat{r}_i\right]
    }{
    \rho\left[(1-\gamma)+\gamma\hat{r}_i\right]
    +
    \varepsilon_{ot}
    }.
\end{aligned}
    \label{eq:adaptive_phi}
\end{equation}
where $\gamma\in[0,1]$ controls the strength of confidence modulation and is set to $1.0$ in all experiments. When $\gamma=0$, Eq.~\eqref{eq:adaptive_phi} reduces to $\phi_i^{a}=\rho/(\rho+\varepsilon_{ot})=\phi$, exactly matching plain unbalanced OT. With the adopted setting $\gamma=1$, we have $\rho_i^{a}=\rho\hat{r}_i$ and $\phi_i^{a}=\rho\hat{r}_i/(\rho\hat{r}_i+\varepsilon_{ot})$. Thus, high-confidence clusters with $\hat{r}_i>1$ receive stronger marginal constraints, whereas low-confidence clusters with $\hat{r}_i<1$ are relaxed further. We denote $\boldsymbol{\phi}^{a}=[\phi_1^{a},\ldots,\phi_K^{a}]^{\top}$. The target side uses the base coefficient $\phi$.

When $\rho$ and all $\rho_i^a$ tend to infinity, the solution approaches entropy-regularized balanced OT. With $\gamma=1$, however, a zero-confidence cluster remains unpenalized for any $\rho$.

For fixed similarities and confidence coefficients, we numerically approximate the minimizer of the entropy-regularized problem
\[
\begin{aligned}
\min_{\mathbf P\ge0}\quad
&\sum_{i,j}P_{ij}c_{ij}
+\varepsilon_{ot}\sum_{i,j}P_{ij}(\log P_{ij}-1)\\
&+\sum_i\rho_i^a\,d\!\left((\mathbf P\mathbf 1)_i\mid a_{m,i}^{adv}\right)\\
&+\rho\sum_j d\!\left((\mathbf P^\top\mathbf 1)_j\mid a_{m,j}^{t}\right),
\end{aligned}
\]
where $d(x\mid y)=x\log(x/y)-x+y$ is the scalar generalized KL divergence for $y>0$, with $0\log(0/y)=0$. The entropy uses the same zero-mass convention. Confidence coefficients are fixed during each transport solve.

We retain the convention that rows correspond to adversarial clusters and columns to target clusters. The implementation stores the transposed matrix, with target clusters first. Under the notation used here, it initializes $\mathbf u^{(0)}=\mathbf v^{(0)}=\mathbf 1$ and updates target-side scaling before adversarial-side scaling:
\begin{equation}
\begin{aligned}
    \mathbf{v}^{(\ell+1)}
    &=\left(\mathbf{a}_{m}^{t}\oslash
       (\mathbf{K}_{m}^{\top}\mathbf{u}^{(\ell)})\right)^{\circ\phi},\\
    \mathbf{u}^{(\ell+1)}
    &=\left(\mathbf{a}_{m}^{adv}\oslash
       (\mathbf{K}_{m}\mathbf{v}^{(\ell+1)})\right)^{\circ\boldsymbol{\phi}^{a}},
\end{aligned}
\label{eq:sinkhorn_update}
\end{equation}
where $\oslash$ denotes element-wise division and $^{\circ}$ denotes element-wise exponentiation; $\boldsymbol{\phi}^{a}$ acts element-wise on the adversarial-side row scaling $\mathbf{u}$, while the scalar exponent $\phi$ is applied to every entry of the target-side column scaling $\mathbf{v}$. The solver runs for at most 100 iterations and stops when $K^{-1}\|\mathbf v^{(\ell+1)}-\mathbf v^{(\ell)}\|_1<10^{-2}$. We use the final scaling vectors to form the numerical plan $\mathbf P_m^*$. The resulting transport plan balances matching costs with confidence-dependent penalties on marginal deviations:
\begin{equation}
    \mathbf{P}_{m}^{*}
    =
    \mathrm{diag}(\mathbf{u})\mathbf{K}_{m}\mathrm{diag}(\mathbf{v}).
    \label{eq:transport_plan}
\end{equation}
To compute the local alignment term, the implementation normalizes the transport plan by its total mass:
\[
    \widetilde{\mathbf{P}}_m
    =\frac{\mathbf{P}_m^*}{\sum_{i,j}\mathbf{P}_m^*(i,j)+\eta},
    \qquad \eta=10^{-8}.
\]
The transport plan and its confidence modulation are computed without gradient tracking. The local alignment loss is the normalized transported cosine distance:
\begin{equation}
    \mathcal{L}_{l}^{m}
    =
    \sum_{i,j}\widetilde{\mathbf{P}}_{m}(i,j)
    \left(1-\mathbf{S}_{m}(i,j)\right).
    \label{eq:local_score}
\end{equation}
The loss in Eq.~\eqref{eq:local_score} measures local matching cost using the normalized plan. The entropy and marginal penalties determine the raw plan $\mathbf{P}_m^*$ and are not added to this loss. Gradients of the local loss pass through the similarities while the transport weights are held fixed. Thus, the normalized plan weights the pairwise similarity gradients that guide the perturbation update. Confidence-adaptive relaxation changes these weights to accommodate unreliable correspondences; its effect on transferability is evaluated experimentally. 

\par Hence, the final loss of the proposed IAU-FOA can be formulated as a combination of the global semantic alignment loss and the adaptive local transport loss. For the $m$-th surrogate image encoder, the loss is defined as:
\begin{equation}
    \mathcal{J}_{m}
    =
    \mathcal{L}_{g}^{m}
    +
    \lambda_l
    \mathcal{L}_{l}^{m},
    \label{eq:per_view_objective}
\end{equation}
where $\lambda_l$ is a hyperparameter that balances the contribution of local transport alignment. We use the same fixed number of clusters for adversarial and target patch features throughout optimization, with $K=10$ by default. The sensitivity to $K$ is evaluated in Fig.~\ref{fig:k_semsitivity}.

\subsection{Dynamic Ensemble Model Weighting}
\label{sec:dynamic_weighting}
To improve adversarial transferability, adversarial examples are generated using the ensemble loss on several surrogate encoders. It is computed as follows:
\begin{equation}
    \mathcal{J}
    =
    \sum_{m=1}^{M}
    \omega_m
    \mathcal{J}_{m}.
    \label{eq:iteration_objective}
\end{equation}
where $\mathcal{J}_{m}$ is the distance loss in Eq.~\eqref{eq:per_view_objective}. Dynamic weights are computed from alignment similarity scores, which are distinct from these distance losses. For the $k$-th evaluation of the ensemble objective, define
\begin{equation}
\begin{aligned}
    A_m^{(k)}
    &=\cos\!\left(\hat{\mathbf{z}}_m^{g}(\mathbf{x}^{c}),
                   \hat{\mathbf{z}}_m^{g}(\mathbf{x}_{tar})\right)\\
    &\quad+\lambda_l\sum_{i,j}\widetilde{\mathbf{P}}_m(i,j)\mathbf{S}_m(i,j),
\end{aligned}
\label{eq:alignment_score}
\end{equation}
where $\mathbf{x}^{c}$ is the current augmented input and the local similarities and transport plan are evaluated for this input. Larger $A_m^{(k)}$ indicates stronger feature alignment. The dynamic weights are calculated from the score ratios defined below. We compute the similarity-score ratio between consecutive evaluations as
\begin{equation}
    q_m^{(k)}
    =\frac{A_m^{(k)}}{A_m^{(k-1)}+\eta},
    \label{eq:progress_ratio}
\end{equation}
where $\eta=10^{-8}$ stabilizes the denominator. For the first evaluation, we set $A_m^{(0)}=A_m^{(1)}$. Here, $k$ indexes successive objective evaluations, including evaluations of different augmented views within the same attack iteration. The ratio reflects score changes across evaluations, which can also arise from differences between augmented views. The weights are
\begin{equation}
    \omega_m^{(k)}
    =M\cdot\frac{\exp(q_m^{(k)}/\tau_w)}
    {\sum_{n=1}^{M}\exp(q_n^{(k)}/\tau_w)},
    \label{eq:dynamic_weight}
\end{equation}
where $\tau_w=1.0$ and the factor $M$ ensures that the weights sum to $M$, matching the scale of equal weighting. For positive scores with $A_m^{(k-1)}$ well above $\eta$, a larger ratio indicates a larger current score relative to the preceding evaluation and receives a larger weight. This rule adapts the ensemble to observed score changes. The complete optimization procedure is summarized in Algorithm~\ref{alg:iaufoa} in the Appendix.
\subsection{Transferability Analysis}
\label{sec:transfer_analysis}
Prior studies relate adversarial transferability to gradient similarity and model smoothness~\citep{yang2021trs}, knowledge transferability~\citep{liang2021uncovering}, and ensemble transferability error~\citep{yao2024understanding}. We compare balanced OT with uniform local matching at a fixed feasible perturbation $\boldsymbol\delta$, holding the encoder, augmentation view, target features, and cluster assignments fixed.

Write $S_{ij}=\mathbf S_m(i,j)$ and $c_{ij}=1-S_{ij}$. Assume that the pairwise similarities are differentiable at $\boldsymbol\delta$. Let $\mathcal J_v$ be a differentiable victim alignment objective, with larger values indicating stronger target alignment, and set $\mathbf g_v=\nabla\mathcal J_v(\boldsymbol\delta)$ and $\mathbf g_{ij}=\nabla_{\boldsymbol\delta}S_{ij}$. For a nonnegative unit-mass plan $\mathbf Q$ held fixed during differentiation, define
\begin{equation}
    \mathbf G(\mathbf Q):=\sum_{i,j}Q_{ij}\mathbf g_{ij}.
    \label{eq:ot_local_direction}
\end{equation}
Here, $\mathbf G(\mathbf Q)$ is the negative gradient of the local alignment loss with the plan held fixed. The uniform reference plan is $Q_{0,ij}=1/K^2$.

\begin{assumption}[Surrogate cost separation]
    \label{ass:1}
    The $K\geq2$ clusters admit a semantic correspondence permutation $\pi$, specified independently of the transport solution. The surrogate costs distinguish these correspondences with a gap $\Delta>0$:
    \begin{equation}
        c_{ij}\geq c_{i\pi(i)}+\Delta,\qquad j\neq\pi(i).
        \label{eq:ot_cost_gap}
    \end{equation}
\end{assumption}

\begin{assumption}[Cross-model gradient separation]
    \label{ass:2}
    There exist $A>0$, $B\geq0$, and deviations $|e_{ij}|\leq\xi$ with $\xi\geq0$ such that
    \begin{equation}
        \langle\mathbf g_v,\mathbf g_{ij}\rangle
        =\begin{cases}
            A+e_{ij},&j=\pi(i),\\
            -B+e_{ij},&j\neq\pi(i).
        \end{cases}
        \label{eq:ot_gradient_separation}
    \end{equation}
\end{assumption}

Define the noncorresponding mass
\begin{equation}
    q(\mathbf Q)=\sum_i\sum_{j\neq\pi(i)}Q_{ij},
    \qquad q_0:=q(\mathbf Q_0)=1-\frac1K.
    \label{eq:ot_noncorresponding_mass}
\end{equation}
\begin{theorem}
    \label{thm:ot_transfer}
Let $H(\mathbf Q)=-\sum_{i,j}Q_{ij}\log Q_{ij}$, with $0\log0=0$. For $\varepsilon_{ot}\geq0$, let $\mathbf Q_{ot}$ minimize
\begin{equation}
    \sum_{i,j}Q_{ij}c_{ij}-\varepsilon_{ot}H(\mathbf Q)
    \label{eq:ot_theory_objective}
\end{equation}
over nonnegative balanced plans satisfying $\sum_jQ_{ij}=1/K$ for every $i$ and $\sum_iQ_{ij}=1/K$ for every $j$.

    Under Assumption~\ref{ass:1},
    \begin{equation}
        q(\mathbf Q_{ot})\leq\bar q
        :=\min\!\left\{1,\frac{\varepsilon_{ot}\log K}{\Delta}\right\}.
        \label{eq:ot_mismatch_bound}
    \end{equation}
    If Assumption~\ref{ass:2} also holds and
    \begin{equation}
        \frac{\varepsilon_{ot}\log K}{\Delta}
        +\frac{2\xi}{A+B}<1-\frac{1}{K},
        \label{eq:ot_joint_transfer_condition}
    \end{equation}
    then
    \begin{equation}
        \langle\mathbf g_v,\mathbf G(\mathbf Q_{ot})-\mathbf G(\mathbf Q_0)\rangle
        \geq\kappa,
        \label{eq:ot_direction_advantage}
    \end{equation}
    where
    \begin{equation}
        \kappa:=(A+B)(q_0-\bar q)-2\xi>0.
        \label{eq:ot_direction_margin}
    \end{equation}
    Under these conditions, let
    \begin{equation}
        \mathbf d_r=\mathbf h+\lambda_l\mathbf G(\mathbf Q_r),
        \quad r\in\{ot,0\},
        \label{eq:ot_update_direction}
    \end{equation}
    where $\mathbf h$ is the common gradient contribution from the other branches and $\lambda_l>0$ is the common local weight. Suppose $\|\mathbf d_r\|_2\leq D$ for some $D>0$, and $\mathcal J_v$ has an $L$-Lipschitz gradient along both update segments, with $L\geq0$. Choose $\alpha>0$ so that both updates remain feasible and, when $L>0$, $\alpha<\lambda_l\kappa/(LD^2)$. Then
    \begin{equation}
    \begin{aligned}
        &\mathcal J_v(\boldsymbol\delta+\alpha\mathbf d_{ot})
        -\mathcal J_v(\boldsymbol\delta+\alpha\mathbf d_0)\\
        &\qquad\geq\tau,
    \end{aligned}
        \label{eq:ot_victim_comparison}
    \end{equation}
    where
    \begin{equation}
        \tau:=\alpha\lambda_l\kappa-L\alpha^2D^2>0.
        \label{eq:ot_finite_step_margin}
    \end{equation}
\end{theorem}

Thus, under the stated assumptions and step-size condition, OT produces stronger target alignment on the victim model than uniform local matching after one feasible ordinary gradient step from the same perturbation, with an objective gap of at least $\tau>0$. This conditional advantage suggests that OT-based local alignment can improve the effectiveness of adversarial examples for targeted transfer, consistent with the empirical gains of the complete attack reported in Section~\ref{sec:experiments}. The proof and discussion of projected sign updates are provided in the Appendix (Section~\ref{sec:appendix_theory}).


\section{Experiments}
\label{sec:experiments}
\subsection{Experimental Setup}
\par \noindent \textbf{Datasets.} Following prior studies~\citep{dong2023robust,li2026frustratingly,jia2026adversarial}, we adopt 1,000 clean source images resized to $224 \times 224 \times 3$ from the NIPS 2017 Adversarial Attacks and Defenses Competition dataset\footnote{\url{https://nips.cc/Conferences/2017/CompetitionTrack}}. For the targeted setting, 1,000 images are randomly sampled from the MSCOCO validation set~\citep{lin2014microsoft} as target images.

\par \noindent \textbf{Implementation Settings.} Following~\citep{li2026multi,wang2026adversarial}, we generate adversarial examples using five surrogate models: three CLIP variants, ViT-B/16~\citep{radford2021learning}, ViT-B/32~\citep{radford2021learning}, and ViT-G/14 (LAION-2B)~\citep{radford2021learning}, together with InternVL3-1B~\citep{zhu2025internvl3} and DINOv2~\citep{oquab2024dinov}. For IAU-FOA, unless otherwise specified, the perturbation budget is $\epsilon=16/255$ under the $\ell_{\infty}$ norm, with a step size of $1/255$ and 300 attack iterations. The number of local clusters is fixed at $K=10$ throughout optimization. Local cluster reference masses are uniform, and the local-loss weight is $\lambda_l=2.0$. Adaptive UOT uses $\rho=0.2$, $\varepsilon_{ot}=0.1$, and $\gamma=1$, with mean-normalized confidence and total-mass normalization of the transport plan. Dynamic weighting uses $\tau_w=1.0$ and a numerical stabilizer of $10^{-8}$. VIA uses intensity scales $\{0.25,0.5,1,2,4\}$ ($h=2$) and a channel-gain amplitude of $s_{\beta}=0.5$. Our main comparison covers twelve MLLMs, comprising six open-source models (Qwen2.5-VL-7B/72B, InternVL3-14B/78B, MiniCPM-o-2.6, LLaVA-1.6-7B) and six closed-source models (Claude-Sonnet-4.6/Opus-4.8, GPT-5.4/5.5, Gemini-3.1-Flash-Lite/3.5-Flash). An additional evaluation on five recent commercial MLLMs brings the total to seventeen targets. For all evaluations, the prompt is fixed as ``Describe this image, no longer than 25 words''. Experiments are conducted on an Ubuntu system equipped with an NVIDIA H100 Tensor Core GPU with 80~GB of GPU memory. 

\par \noindent \textbf{Competitive Methods.}
We compare the proposed IAU-FOA with seven representative targeted transfer attacks for MLLMs, including AdvDiffVLM~\citep{guo2024efficient}, AnyAttack~\citep{zhang2025anyattack}, COA~\citep{xie2025chain}, M-Attack~\citep{li2026frustratingly}, FOA-Attack~\citep{jia2026adversarial}, M-Attack-v2~\citep{zhao2026pushing}, and MPCAttack~\citep{li2026multi}. For baselines other than AdvDiffVLM, we use the same surrogate models, perturbation budget, and attack iterations as IAU-FOA; the remaining method-specific settings follow their original papers. AdvDiffVLM follows the settings of its original paper~\citep{guo2024efficient} and is not constrained by the $\ell_{\infty}$ budget of $16/255$. All methods use the same evaluation protocol described below.

\par \noindent \textbf{Evaluation metrics.} Following~\citep{li2026frustratingly,jia2026adversarial}, we evaluate targeted transferability with an LLM-as-a-judge protocol. For each adversarial--target pair, the target MLLM first produces captions for both images, and GPTScore is then used to quantify their semantic similarity by using GPT-4o. We regard an attack as successful when the similarity score is at least 0.5, indicating that the adversarial image is judged to share the target subject. In addition to the attack success rate (ASR), we report the average similarity score (AvgSim) to reflect overall target alignment; higher values of both metrics indicate better targeted transferability.


\subsection{Hyper-parameter Selection}
\begin{figure*}[t]
    \centering
    \includegraphics[width=\linewidth]{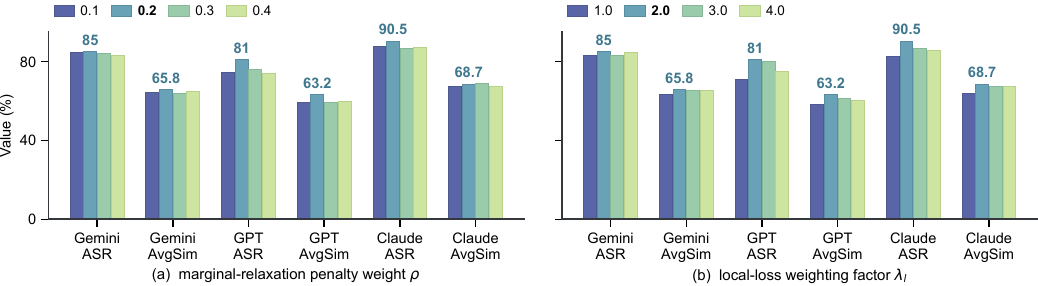}
    \vspace{-7mm}
    \caption{Sensitivity to (a) the marginal-relaxation penalty $\rho$ and (b) the local-loss weight $\lambda_l$, evaluated on 100 source--target image pairs per model. ASR and AvgSim are averaged over two models per family (Gemini, GPT, and Claude). Bold legend entries mark the selected settings, $\rho=0.2$ and $\lambda_l=2.0$; numerical labels show metric values.}
    \label{fig:params}
\end{figure*}

\begin{figure}[t]
\begin{center}
 \includegraphics[width=1.0\linewidth]{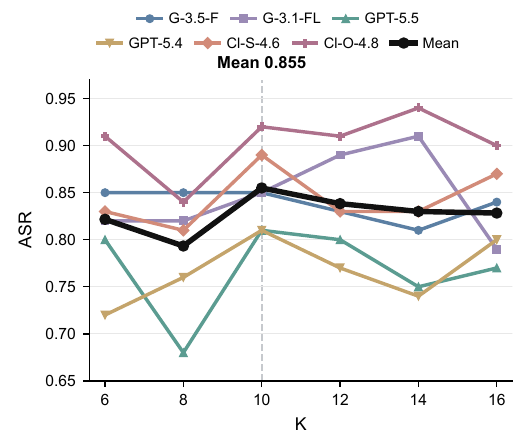}
\end{center}
\vspace{-7mm}
\caption{Sensitivity to the number of local clusters $K$, evaluated on 100 source--target image pairs per model. Colored curves show ASR on six commercial MLLMs; the thick black curve shows their mean. The dashed line marks the default $K=10$, which maximizes mean ASR among the tested values.}
        \label{fig:k_semsitivity}
\end{figure}

All results in Figs.~\ref{fig:params} and~\ref{fig:k_semsitivity} are evaluated on 100 source--target image pairs per target model.
We investigate the sensitivity of IAU-FOA to three hyperparameters: the marginal-relaxation penalty weight $\rho$ in adaptive unbalanced OT, the local-loss weighting factor $\lambda_l$ in the dual-granularity alignment objective, and the number of local clusters $K$. We vary one hyperparameter at a time while keeping all other settings fixed. For $\rho$ and $\lambda_l$, we report the arithmetic means of ASR and AvgSim over two models in each family: Gemini-3.5-Flash and Gemini-3.1-Flash-Lite for Gemini, GPT-5.5 and GPT-5.4 for GPT, and Claude-Sonnet-4.6 and Claude-Opus-4.8 for Claude. For $K$, we report ASR on all six commercial targets in the main evaluation.

As shown in Fig.~\ref{fig:params}(a), we vary $\rho$ over $\{0.1,0.2,0.3,0.4\}$. The setting $\rho=0.2$ yields the highest reported ASR and AvgSim for Gemini, GPT, and Claude, achieving family-mean ASRs of $85.0\%$, $81.0\%$, and $90.5\%$ for Gemini, GPT, and Claude, respectively, with corresponding mean AvgSim values of $65.8\%$, $63.2\%$, and $68.7\%$. These family-mean ASRs agree with the corresponding two-model averages of the full-method results in Table~\ref{tab:ablation}. A smaller value may over-relax the marginal constraints, whereas a larger value moves the transport plan closer to balanced OT and may retain unreliable local correspondences. Fig.~\ref{fig:params}(b) evaluates $\lambda_l\in\{1.0,2.0,3.0,4.0\}$. Setting $\lambda_l=2.0$ also yields the highest reported values of both metrics for Gemini, GPT, and Claude, with the same ASR and AvgSim values reported above. These results favor an intermediate local-loss weight: too little emphasis on local alignment may limit its contribution, whereas too much may reduce the benefit of global semantic alignment. Based on these results, we use $\rho=0.2$ and $\lambda_l=2.0$ in all subsequent experiments.

Figure~\ref{fig:k_semsitivity} further evaluates $K\in\{6,8,10,12,14,16\}$,
which controls the granularity of the clustered patch features used for local
transport alignment. The highest mean ASR is $0.855$ at $K=10$, compared
with $0.822$ and $0.793$ at $K=6$ and $K=8$, respectively. Increasing $K$
beyond 10 yields mean ASRs between $0.828$ and $0.838$, indicating that
finer clustering does not consistently improve aggregate transferability.
The preferred setting nevertheless varies by target: $K=10$ achieves the
highest ASR on GPT-5.5, GPT-5.4, and Cl-S-4.6 and ties for the best result on
G-3.5-F, whereas G-3.1-FL and Cl-O-4.8 peak at $K=14$ with ASRs of $0.910$
and $0.940$, respectively. Based on the best mean performance across the
six targets, we use a fixed $K=10$ throughout optimization unless otherwise
specified.

\begin{table*}[t]
\centering
\footnotesize
\setlength{\tabcolsep}{3.4pt}
\caption{%
  \textbf{Targeted transfer to six open-source MLLMs on the same 1,000 source--target image pairs per model.}
  ASR and AvgSim measure attack success and mean similarity, respectively;
  bold marks column maxima.}
\label{tab:main-open}

\begin{tabular}{c cc cc cc cc cc cc cc}
\toprule

& \multicolumn{12}{c}{Target MLLMs}
& \multicolumn{2}{c}{Average} \\
\cmidrule(lr){2-13}\cmidrule(l){14-15}

& \multicolumn{2}{c}{Qwen2.5-VL-7B}
& \multicolumn{2}{c}{InternVL3-14B}
& \multicolumn{2}{c}{MiniCPM-o-2.6 (8B)}
& \multicolumn{2}{c}{LLaVA-1.6-7B}
& \multicolumn{2}{c}{Qwen2.5-VL-72B}
& \multicolumn{2}{c}{InternVL3-78B}
& \multicolumn{2}{c}{Overall} \\

Method & ASR & AvgSim
& ASR & AvgSim
& ASR & AvgSim
& ASR & AvgSim
& ASR & AvgSim
& ASR & AvgSim
& ASR & AvgSim \\

\midrule

COA~\citep{xie2025chain}
& 0.001 & 0.018
& 0.002 & 0.018
& 0.002 & 0.020
& 0.004 & 0.016
& 0.001 & 0.020
& 0.000 & 0.015
& 0.002 & 0.018 \\

AnyAttack~\citep{zhang2025anyattack}
& 0.011 & 0.039
& 0.011 & 0.032
& 0.014 & 0.036
& 0.023 & 0.036
& 0.005 & 0.042
& 0.008 & 0.028
& 0.012 & 0.036 \\

AdvDiffVLM~\citep{guo2024efficient}
& 0.309 & 0.313
& 0.315 & 0.318
& 0.293 & 0.292
& 0.309 & 0.296
& 0.314 & 0.327
& 0.194 & 0.239
& 0.289 & 0.298 \\

M-Attack~\citep{li2026frustratingly}
& 0.569 & 0.472
& 0.661 & 0.523
& 0.632 & 0.520
& 0.680 & 0.562
& 0.571 & 0.476
& 0.440 & 0.390
& 0.592 & 0.491 \\

FOA~\citep{jia2026adversarial}
& 0.575 & 0.488
& 0.703 & 0.561
& 0.692 & 0.551
& 0.725 & 0.598
& 0.617 & 0.506
& 0.490 & 0.422
& 0.634 & 0.521 \\

M-Attack-v2~\citep{zhao2026pushing}
& 0.705 & 0.564
& 0.809 & 0.631
& 0.814 & 0.634
& 0.812 & 0.666
& 0.693 & 0.563
& 0.648 & 0.535
& 0.747 & 0.599 \\

MPCAttack~\citep{li2026multi}
& 0.697 & 0.559
& 0.853 & 0.659
& 0.836 & 0.649
& 0.839 & 0.686
& 0.714 & 0.577
& 0.772 & 0.606
& 0.785 & 0.623 \\

\midrule

\textbf{IAU-FOA (Ours)}
& \textbf{0.821} & \textbf{0.645}
& \textbf{0.910} & \textbf{0.710}
& \textbf{0.866} & \textbf{0.692}
& \textbf{0.885} & \textbf{0.723}
& \textbf{0.823} & \textbf{0.653}
& \textbf{0.871} & \textbf{0.685}
& \textbf{0.863} & \textbf{0.685} \\

\bottomrule
\end{tabular}
\end{table*}

\begin{table*}[t]
\centering
\footnotesize
\setlength{\tabcolsep}{3.4pt}
\caption{%
  \textbf{Targeted transfer to six commercial MLLMs on the same 1,000 source--target image pairs per model.}
  ASR and AvgSim measure attack success and mean similarity, respectively;
  bold marks column maxima.}
\label{tab:main-closed}

\begin{tabular}{c cc cc cc cc cc cc cc}
\toprule

& \multicolumn{12}{c}{Target MLLMs}
& \multicolumn{2}{c}{Average} \\
\cmidrule(lr){2-13}\cmidrule(l){14-15}

& \multicolumn{2}{c}{Gemini-3.5-Flash}
& \multicolumn{2}{c}{Gemini-3.1-Flash-Lite}
& \multicolumn{2}{c}{GPT-5.5}
& \multicolumn{2}{c}{GPT-5.4}
& \multicolumn{2}{c}{Claude-Sonnet-4.6}
& \multicolumn{2}{c}{Claude-Opus-4.8}
& \multicolumn{2}{c}{Overall} \\

Method & ASR & AvgSim
& ASR & AvgSim
& ASR & AvgSim
& ASR & AvgSim
& ASR & AvgSim
& ASR & AvgSim
& ASR & AvgSim \\

\midrule

COA~\citep{xie2025chain}
& 0.001 & 0.010
& 0.000 & 0.013
& 0.002 & 0.019
& 0.001 & 0.018
& 0.000 & 0.013
& 0.002 & 0.018
& 0.001 & 0.015 \\

AnyAttack~\citep{zhang2025anyattack}
& 0.004 & 0.019
& 0.004 & 0.028
& 0.009 & 0.032
& 0.005 & 0.027
& 0.008 & 0.026
& 0.006 & 0.029
& 0.006 & 0.027 \\

AdvDiffVLM~\citep{guo2024efficient}
& 0.055 & 0.129
& 0.068 & 0.127
& 0.116 & 0.171
& 0.112 & 0.172
& 0.277 & 0.297
& 0.324 & 0.313
& 0.159 & 0.202 \\

M-Attack~\citep{li2026frustratingly}
& 0.309 & 0.321
& 0.312 & 0.333
& 0.285 & 0.292
& 0.219 & 0.255
& 0.425 & 0.384
& 0.535 & 0.452
& 0.348 & 0.339 \\

FOA~\citep{jia2026adversarial}
& 0.340 & 0.350
& 0.342 & 0.353
& 0.350 & 0.332
& 0.296 & 0.296
& 0.474 & 0.417
& 0.581 & 0.484
& 0.397 & 0.372 \\

M-Attack-v2~\citep{zhao2026pushing}
& 0.525 & 0.455
& 0.506 & 0.450
& 0.454 & 0.401
& 0.383 & 0.357
& 0.630 & 0.511
& 0.723 & 0.568
& 0.537 & 0.457 \\

MPCAttack~\citep{li2026multi}
& 0.673 & 0.547
& 0.675 & 0.542
& 0.453 & 0.408
& 0.390 & 0.365
& 0.621 & 0.514
& 0.724 & 0.581
& 0.589 & 0.493 \\

\midrule

\textbf{IAU-FOA (Ours)}
& \textbf{0.776} & \textbf{0.616}
& \textbf{0.789} & \textbf{0.622}
& \textbf{0.696} & \textbf{0.561}
& \textbf{0.672} & \textbf{0.553}
& \textbf{0.744} & \textbf{0.605}
& \textbf{0.822} & \textbf{0.641}
& \textbf{0.750} & \textbf{0.600} \\

\bottomrule
\end{tabular}
\end{table*}

\subsection{Comparative Experimental Results}
\label{sec:comparative_results}
\par \noindent \textbf{Comparisons with different attack methods.}
All results in Tables~\ref{tab:main-open} and~\ref{tab:main-closed} are evaluated on the same 1,000 source--target image pairs per target model.
We compare the proposed IAU-FOA with seven targeted transfer attack baselines, including COA, AnyAttack, AdvDiffVLM, M-Attack, FOA, M-Attack-v2, and MPCAttack, across both open-source and closed-source MLLMs. As shown in Table~\ref{tab:main-open}, IAU-FOA achieves the highest ASR and AvgSim on all six open-source targets across the Qwen2.5-VL, InternVL3, MiniCPM-o, and LLaVA model families. Its reported average ASR and AvgSim are $86.3\%$ and $68.5\%$, respectively, compared with $78.5\%$ and $62.3\%$ for MPCAttack, the strongest baseline on average. Compared with the best-performing baseline for each target and metric, IAU-FOA improves ASR by $3.0$--$11.6$ percentage points and AvgSim by $3.7$--$8.1$ percentage points. The largest gains in both metrics occur on Qwen2.5-VL-7B, where IAU-FOA achieves an ASR of $82.1\%$ and an AvgSim of $64.5\%$, compared with $70.5\%$ and $56.4\%$ for M-Attack-v2.

As reported in Table~\ref{tab:main-closed}, IAU-FOA also ranks first in both metrics on all six commercial targets from the Gemini, GPT, and Claude families. Its reported average ASR and AvgSim are $75.0\%$ and $60.0\%$, respectively, compared with $58.9\%$ and $49.3\%$ for MPCAttack. Relative to the strongest competing method for each target and metric, the improvements range from $9.8$ to $28.2$ percentage points in ASR and from $6.0$ to $18.8$ percentage points in AvgSim. The largest gains occur on GPT-5.4, where IAU-FOA improves ASR from $39.0\%$ to $67.2\%$ and AvgSim from $36.5\%$ to $55.3\%$. These results demonstrate that the proposed method outperforms previous methods in both attack success rate and average target-semantic alignment.
\par \noindent \textbf{Evaluation on recent commercial MLLMs.}
In August 2026, we further evaluate IAU-FOA on five recent commercial MLLMs: OpenAI's GPT-5.6-Luna, GPT-5.6-Sol, and GPT-5.6-Terra, Anthropic's Claude-Opus-5, and Google's Gemini-3.7-Flash. These targets span three providers and multiple capability and efficiency tiers. We use the same 1,000 source--target image pairs per target model as the main comparison and compare against M-Attack-v2 and MPCAttack, its two strongest baselines. As shown in Table~\ref{tab:new-targets}, IAU-FOA achieves the highest ASR and AvgSim on all five targets, averaging $66.7\%$ and $54.2\%$, respectively. Compared with MPCAttack, the strongest baseline on average, these results represent gains of $21.5$ percentage points in ASR and $13.4$ points in AvgSim. Across individual targets, improvements over the stronger baseline range from $7.3$ to $30.7$ percentage points in ASR and from $4.0$ to $19.4$ points in AvgSim. The largest ASR gain occurs on GPT-5.6-Luna, increasing from $25.0\%$ to $55.7\%$. On GPT-5.6-Sol and GPT-5.6-Terra, ASR improves by $29.1$ and $29.7$ percentage points over the respective strongest baselines, while the gains on Claude-Opus-5 and Gemini-3.7-Flash are $7.3$ and $9.9$ points. These results demonstrate improved transferability across all five targets, with the largest gains on the GPT-5.6 family.


\begin{table}[t]
\centering
\footnotesize
\setlength{\tabcolsep}{2.2pt}
\renewcommand{\arraystretch}{1.12}
\caption{%
  \textbf{Targeted transfer to five recent commercial MLLMs (August 2026).}
  ASR and AvgSim are evaluated on 1,000 source--target image pairs per model;
  bold marks row maxima.}
\label{tab:new-targets}

\begin{tabular}{@{}ccccc@{}}
\toprule
Target MLLM & Metric
& \makecell{M-Attack-v2\\[-1pt]\citep{zhao2026pushing}}
& \makecell{MPCAttack\\[-1pt]\citep{li2026multi}}
& \makecell{\textbf{IAU-FOA}\\[-1pt]\textbf{(Ours)}} \\
\midrule

\multirow{2}{*}{GPT-5.6-Luna}
& ASR & 0.245 & 0.250 & \textbf{0.557} \\
& AvgSim & 0.275 & 0.277 & \textbf{0.471} \\
\addlinespace[1.5pt]

\multirow{2}{*}{GPT-5.6-Sol}
& ASR & 0.323 & 0.315 & \textbf{0.614} \\
& AvgSim & 0.326 & 0.321 & \textbf{0.509} \\
\addlinespace[1.5pt]

\multirow{2}{*}{GPT-5.6-Terra}
& ASR & 0.296 & 0.299 & \textbf{0.596} \\
& AvgSim & 0.306 & 0.307 & \textbf{0.492} \\
\addlinespace[1.5pt]

\multirow{2}{*}{Claude-Opus-5}
& ASR & 0.669 & 0.683 & \textbf{0.756} \\
& AvgSim & 0.541 & 0.552 & \textbf{0.592} \\
\addlinespace[1.5pt]

\multirow{2}{*}{Gemini-3.7-Flash}
& ASR & 0.553 & 0.713 & \textbf{0.812} \\
& AvgSim & 0.480 & 0.582 & \textbf{0.647} \\

\midrule

\multirow{2}{*}{\textbf{Average}}
& ASR & 0.417 & 0.452 & \textbf{0.667} \\
& AvgSim & 0.385 & 0.408 & \textbf{0.542} \\

\bottomrule
\end{tabular}
\end{table}

\subsection{Evaluation under Adversarial Defenses}
\label{sec:performance_analysis}

\begin{table}[t]
\centering
\footnotesize
\setlength{\tabcolsep}{4pt}
\renewcommand{\arraystretch}{1.08}
\caption{\textbf{Robustness against input defenses: mean ASR across six commercial MLLMs.}
All settings use 100 source--target image pairs per model. Def. Avg. averages the six defended settings; bold marks row maxima.}
\label{tab:defense}
\begin{tabular}{@{}lccc@{}}
\toprule
\textbf{Defense} & \textbf{M-Attack-v2}~\citep{zhao2026pushing} & \textbf{MPCAttack}~\citep{li2026multi}
& \textbf{IAU-FOA(Ours)} \\
\midrule
No defense & 0.638 & 0.690 & \textbf{0.855} \\
\midrule
JPEG$_{75}$~\citep{dziugaite2016study} & 0.557 & 0.640 & \textbf{0.803} \\
BDR$_{4\text{-bit}}$~\citep{xu2017feature} & 0.508 & 0.605 & \textbf{0.778} \\
Median$_{3\times3}$~\citep{xu2017feature} & 0.343 & 0.382 & \textbf{0.605} \\
TVM~\citep{guo2017countering} & 0.495 & 0.567 & \textbf{0.737} \\
ComDefend~\citep{jia2019comdefend} & 0.370 & 0.393 & \textbf{0.618} \\
DISCO~\citep{ho2022disco} & 0.297 & 0.382 & \textbf{0.587} \\
\midrule
\textbf{Def. Avg.} & 0.428 & 0.495 & \textbf{0.688} \\
\bottomrule
\end{tabular}
\end{table}

We conduct comparative experiments under six input-transformation and purification
defenses: JPEG compression~\citep{dziugaite2016study}, total-variation minimization (TVM)~\citep{guo2017countering}, bit-depth
reduction (BDR)~\citep{xu2017feature}, median filtering~\citep{xu2017feature}, ComDefend~\citep{jia2019comdefend}, and DISCO~\citep{ho2022disco}. JPEG compression
uses a quality factor of 75, BDR reduces the color depth to 4 bits, and median
filtering uses a $3\times3$ kernel. All settings, including the undefended setting,
are evaluated on 100 source--target image pairs per target model. Table~\ref{tab:defense} reports the mean
ASR across the six commercial MLLMs for each defense; Def. Avg. averages the
six defended settings, excluding the undefended setting. IAU-FOA achieves the highest ASR in every defended setting,
with an average of $0.688$, compared with $0.495$ for MPCAttack and $0.428$
for M-Attack-v2. These results correspond to average improvements of $19.3$
and $26.0$ percentage points, respectively. The gains over MPCAttack range
from $16.3$ to $22.5$ percentage points across individual defenses.
Under DISCO, IAU-FOA achieves an ASR of
$0.587$, exceeding MPCAttack ($0.382$) and M-Attack-v2 ($0.297$).
These results show that the advantage of the proposed method persists after both conventional
preprocessing and learned purification.

\subsection{Ablation Study}
\label{sec:ablation_study}
To evaluate the contribution of each component in IAU-FOA, we compare the full method with four ablated variants: w/o OT removes the local alignment branch and retains only global alignment; w/o Adaptive UOT replaces adaptive unbalanced transport with balanced OT; w/o VIA disables visual-invariance augmentation; and w/o Dynamic ensemble fixes all surrogate weights to one. All variants use the same 100 source--target image pairs per target model against six commercial MLLMs, with all remaining attack and evaluation settings unchanged. Table~\ref{tab:ablation} reports per-model ASR and the mean ASR and AvgSim across the six targets.


\begin{table*}[t]
\centering
\small
\setlength{\tabcolsep}{5.5pt}
\caption{%
    \textbf{Component ablation on six commercial MLLMs using the same 100 source--target image pairs per model.}
    G/Cl denote Gemini/Claude; F/FL/S/O denote Flash/Flash-Lite/Sonnet/Opus.
    w/o OT removes local alignment; w/o Adaptive UOT uses BOT;
    w/o Dynamic ensemble uses Equal weighting.}
\label{tab:ablation}

\begin{tabular}{l cccccc cc}
\toprule
& \multicolumn{6}{c}{Attack success rate per target MLLM}
& \multicolumn{2}{c}{Overall} \\
\cmidrule(lr){2-7}\cmidrule(l){8-9}

Variant
& G-3.5-F
& G-3.1-FL
& GPT-5.5
& GPT-5.4
& Cl-S-4.6
& Cl-O-4.8
& ASR
& AvgSim \\

\midrule

\textbf{IAU-FOA (Ours)}
& 0.850
& 0.850
& 0.810
& 0.810
& 0.890
& 0.920
& \textbf{0.855}
& \textbf{0.659} \\

\midrule

\quad w/o OT
& 0.680
& 0.730
& 0.650
& 0.550
& 0.760
& 0.780
& 0.692
& 0.543 \\

\quad w/o Adaptive UOT
& 0.840
& 0.800
& 0.780
& 0.750
& 0.850
& 0.850
& 0.812
& 0.635 \\

\quad w/o VIA
& 0.770
& 0.750
& 0.610
& 0.500
& 0.630
& 0.790
& 0.675
& 0.551 \\

\quad w/o Dynamic ensemble
& 0.850
& 0.800
& 0.770
& 0.730
& 0.830
& 0.820
& 0.800
& 0.646 \\

\bottomrule
\end{tabular}
\end{table*}

The full method achieves the highest mean ASR and AvgSim of $0.855$ and $0.659$, respectively. Removing local OT alignment reduces these metrics by $16.3$ and $11.6$ percentage points, while replacing Adaptive UOT with balanced OT reduces them by $4.3$ and $2.4$ points, supporting both local matching and confidence-adaptive transport. Removing VIA causes the largest ASR decrease of $18.0$ points, together with a $10.8$-point decrease in AvgSim. Equal ensemble weighting reduces mean ASR and AvgSim by $5.5$ and $1.3$ points, respectively; ASR decreases on five targets and remains unchanged on G-3.5-F. Overall, all four ablations reduce average performance, with VIA contributing most to ASR and local OT alignment contributing most to AvgSim under the evaluated settings.

\begin{figure}[t]
\begin{center}
 \includegraphics[width=1.0\linewidth]{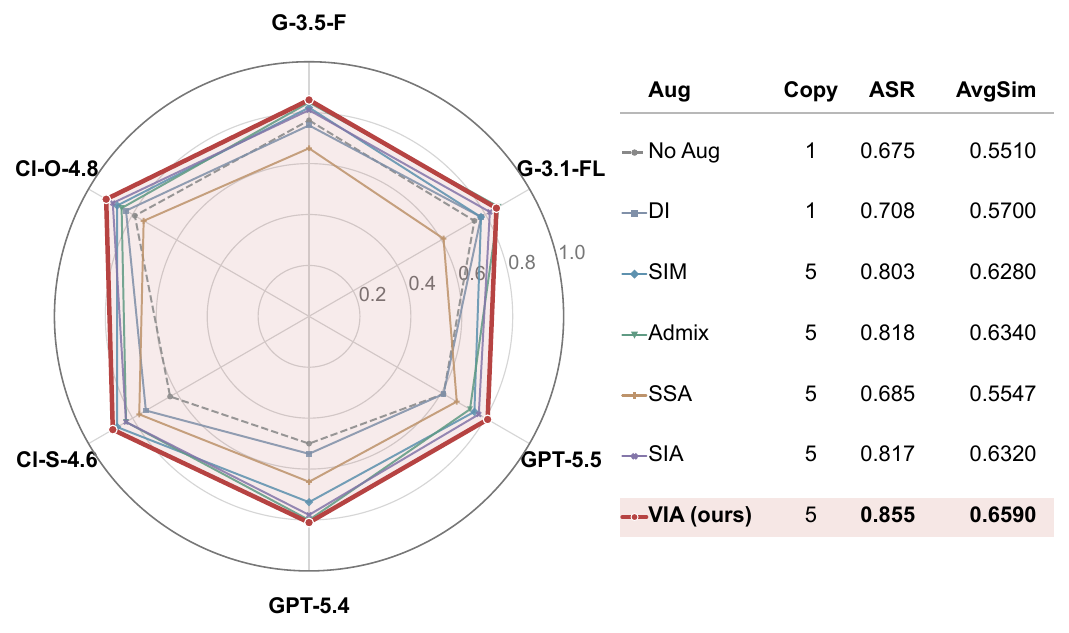}
\end{center}
\caption{Input augmentation strategies on six commercial MLLMs, evaluated on 100 source--target image pairs per model. The radar chart shows ASR on Gemini-3.5-Flash (G-3.5-F), Gemini-3.1-Flash-Lite (G-3.1-FL), GPT-5.5, GPT-5.4, Claude-Sonnet-4.6 (Cl-S-4.6), and Claude-Opus-4.8 (Cl-O-4.8). The table reports input copies (Copy) and mean ASR and AvgSim across targets.}
        \label{fig:aug_radar}
\end{figure}

\subsection{Analysis of the Proposed Method}
\label{sec:component_analysis}
We conduct all comparative experiments in this subsection on 100 source--target image pairs per model.
\par \noindent \textbf{Visual-Invariance Augmentation.}
As shown in Fig.~\ref{fig:aug_radar}, we compare the proposed visual-invariance augmentation (VIA) with No Aug and five input-diversification strategies (DI~\citep{xie2019improving}, SIM~\citep{lin2019nesterov}, Admix~\citep{wang2021admix}, SSA~\citep{long2022frequency}, and SIA~\citep{wang2023structure}) on six commercial MLLMs, while keeping all other attack settings fixed. VIA achieves the highest ASR on all six targets and the best aggregate performance, with an average ASR of $0.855$ and an AvgSim of $0.6590$. Compared with No Aug ($0.675$ ASR and $0.5510$ AvgSim), it improves the two metrics by $18.0$ and $10.80$ percentage points, respectively. Admix is the strongest competing augmentation on average, with an ASR of $0.818$ and an AvgSim of $0.6340$; VIA exceeds these scores by $3.7$ and $2.50$ percentage points. SIM, Admix, SSA, and SIA all use five transformed copies, matching VIA, yet obtain lower average ASR and AvgSim. This comparison indicates that the gains depend on the augmentation strategy, beyond the number of augmented views. VIA varies intensity and channel-wise gains while preserving spatial content, encouraging the perturbation to remain effective under exposure, illumination-color, and white-balance variations. The improvements across the Gemini, GPT, and Claude families support the effectiveness of these photometric transformations for targeted transferability.

\begin{figure}[t]
\begin{center}
 \includegraphics[width=0.9\linewidth]{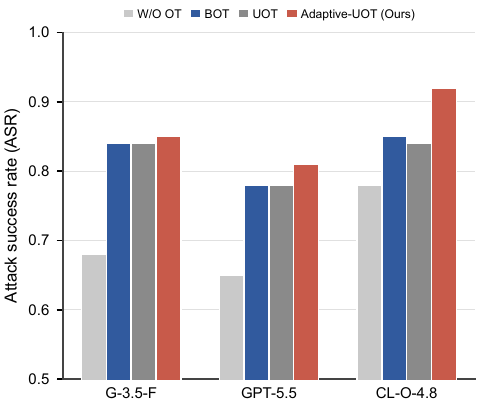}
\end{center}
\vspace{-6mm}
\caption{OT-based local alignment on three commercial MLLMs, evaluated on 100 source--target image pairs per model. W/O OT removes the local alignment branch; BOT and UOT use balanced and plain unbalanced OT, respectively. Adaptive-UOT is our confidence-aware formulation and achieves the highest ASR on all three targets.}
        \label{fig:ot_asr_comparison}
\end{figure}

\par \noindent \textbf{Adaptive Unbalanced Optimal Transport.}
As shown in Fig.~\ref{fig:ot_asr_comparison}, we compare Adaptive-UOT with
W/O OT, BOT, and UOT while keeping all other attack settings fixed.
W/O OT removes the local alignment branch, and BOT imposes balanced transport
constraints; these correspond to w/o OT and w/o Adaptive UOT in
Table~\ref{tab:ablation}. UOT relaxes the marginal constraints without
confidence-aware adaptation. Adaptive-UOT achieves the highest ASR on all
three targets, reaching $0.85$, $0.81$, and $0.92$ on G-3.5-F, GPT-5.5,
and Cl-O-4.8, respectively. Its mean ASR is $0.860$, compared with $0.703$,
$0.823$, and $0.820$ for W/O OT, BOT, and UOT, corresponding to gains of
$15.7$, $3.7$, and $4.0$ percentage points. All transport formulations
outperform W/O OT on every target, supporting OT-based local alignment.
Plain UOT matches BOT on G-3.5-F and GPT-5.5 but performs slightly worse
on Cl-O-4.8 ($0.84$ vs.\ $0.85$). Adaptive-UOT's largest gains occur on
Cl-O-4.8: $7.0$ percentage points over BOT and $8.0$ over UOT.
These results support adapting marginal regularization to matching confidence;
uniform relaxation alone yields no improvement over BOT here.
These experimental results demonstrate the effectiveness of the proposed
Adaptive Unbalanced Optimal Transport in improving targeted transferability.

\begin{table}[t]
\centering
\footnotesize
\setlength{\tabcolsep}{1.7pt}
\caption{%
  \textbf{ASR comparison of ensemble weighting strategies.}
  Results use 100 source--target image pairs per model; only the weighting rule varies.
  G/Cl denote Gemini/Claude; F/FL/S/O denote Flash/Flash-Lite/Sonnet/Opus.
  Bold marks column maxima.}
\label{tab:ensemble-weighting}

\begin{tabular*}{\linewidth}{@{\extracolsep{\fill}}c cccccc c@{}}
\toprule
Weighting
& \makecell{G-3.5-\\F}
& \makecell{G-3.1-\\FL}
& \makecell{GPT-\\5.5}
& \makecell{GPT-\\5.4}
& \makecell{Cl-S-\\4.6}
& \makecell{Cl-O-\\4.8}
& ASR \\

\midrule

Equal
& \textbf{0.850}
& 0.800
& 0.770
& 0.730
& 0.830
& 0.820
& 0.800 \\

Random
& 0.810
& 0.800
& 0.760
& 0.770
& \textbf{0.890}
& 0.870
& 0.817 \\

Score-Prop.
& \textbf{0.850}
& 0.830
& 0.800
& 0.740
& 0.860
& 0.910
& 0.832 \\

\textbf{Dynamic (Ours)}
& \textbf{0.850}
& \textbf{0.850}
& \textbf{0.810}
& \textbf{0.810}
& \textbf{0.890}
& \textbf{0.920}
& \textbf{0.855} \\

\bottomrule
\end{tabular*}
\end{table}

\par \noindent \textbf{Dynamic Ensemble Model Weighting.}
As shown in Table~\ref{tab:ensemble-weighting}, we compare dynamic ensemble weighting with three alternatives while keeping all other attack settings fixed. Using the score notation in Sec.~\ref{sec:dynamic_weighting}, Equal Weighting sets every weight to one, and Random Weighting samples symmetric Dirichlet weights scaled to sum to $M$. Score-Proportional Weighting assigns $\omega_m^{(k)}\propto\max(A_m^{(k)},0)$ with the same normalization, reverting to equal weights when the sum of nonnegative scores is at most $10^{-8}$. Thus, this control weights encoders by their current nonnegative alignment scores. Dynamic weighting instead uses the current-to-history score ratio in Eqs.~\eqref{eq:progress_ratio} and~\eqref{eq:dynamic_weight}.
Dynamic weighting achieves the highest mean ASR of $0.855$ across the six
target models, compared with $0.800$, $0.817$, and $0.832$ for Equal, Random,
and Score-Proportional Weighting, respectively. These results correspond to
gains of $5.5$, $3.8$, and $2.3$ percentage points.
It ranks first on all six targets, with outright best results on G-3.1-FL,
GPT-5.5, GPT-5.4, and Cl-O-4.8, and ties on G-3.5-F and Cl-S-4.6.
The largest gain over Score-Proportional Weighting occurs on GPT-5.4,
where ASR increases from $0.740$ to $0.810$.
These experimental results demonstrate the effectiveness of the proposed
Dynamic Ensemble Model Weighting in improving targeted transferability. We further analyze visual fidelity and conduct experiments to examine the effect
of the number of attack iterations, with results provided in the Appendix
(Figs.~\ref{fig:visualization} and~\ref{fig:iteration}). We also examine the association
between source--target CLIP similarity and attack success, with results
reported in the Appendix (Table~\ref{tab:clip_similarity}).

\section{Conclusion}
In this paper, we investigate targeted adversarial transferability across
heterogeneous MLLMs and propose IAU-FOA to improve black-box attack performance
through global and local feature alignment. Specifically, global alignment
provides semantic guidance, while optimal transport matches clustered patch
features to exploit fine-grained visual structures. To address unreliable local
correspondences under balanced transport constraints, we propose adaptive
unbalanced optimal transport, which adjusts marginal regularization according
to matching confidence. We also provide a conditional theoretical analysis of
how local transport alignment can improve targeted transfer. Furthermore, we
develop a dynamic ensemble weighting strategy that adjusts each surrogate
encoder's contribution according to relative changes in its alignment
similarity score. Finally, we propose a visual-invariance augmentation strategy
that combines bidirectional intensity scaling and channel-wise white-balance
adjustment to improve transferability under varying photometric conditions.
Extensive experiments on open-source and closed-source MLLMs demonstrate that
IAU-FOA outperforms the compared transferable attack methods. Ablation studies
and component analyses further demonstrate the effectiveness of the proposed
designs.

\ifCLASSOPTIONcaptionsoff
  \newpage
\fi
\bibliographystyle{plain}
\bibliography{refer}

\onecolumn
\setcounter{section}{0}
\setcounter{subsection}{0}
\setcounter{equation}{0}
\setcounter{figure}{0}
\setcounter{table}{0}
\setcounter{algorithm}{0}
\renewcommand{\thesection}{S\arabic{section}}
\renewcommand{\theequation}{S\arabic{equation}}
\renewcommand{\thefigure}{S\arabic{figure}}
\renewcommand{\thetable}{S\arabic{table}}
\renewcommand{\thealgorithm}{S\arabic{algorithm}}
\renewcommand{\theHsection}{S\arabic{section}}
\renewcommand{\theHequation}{S\arabic{equation}}
\renewcommand{\theHfigure}{S\arabic{figure}}
\renewcommand{\theHtable}{S\arabic{table}}
\providecommand{\theHalgorithm}{}
\renewcommand{\theHalgorithm}{S\arabic{algorithm}}
\title{Visual-Invariance-Augmented Feature Optimal Alignment for Transferable Adversarial Attacks against Closed-Source MLLMs\\Appendix}
\author{Xiaojun Jia, Simeng Qin, Yiming Li, Jie Liao, Sensen Gao,\\
Ke Ma, Yang Liu, and Xiaochun Cao}
\markboth{Appendix}{Visual-Invariance-Augmented Feature Optimal Alignment}
\makeatletter
\let\maketitle\supplementmaketitle
\let\@maketitle\supplementinternalmaketitle
\makeatother
\maketitle
\section{Proofs and Details of the Transferability Analysis}
\label{sec:appendix_theory}
This section proves Theorem~\ref{thm:ot_transfer} and discusses projected sign updates. Throughout the local comparison, the encoder, augmentation view, target features, cluster assignments, and transport weights are held fixed. In particular, the frozen-plan local loss satisfies
\begin{equation}
    -\nabla_{\boldsymbol\delta}\mathcal L_l(\boldsymbol\delta;\mathbf Q)
    =-\nabla_{\boldsymbol\delta}\sum_{i,j}Q_{ij}(1-S_{ij})
    =\sum_{i,j}Q_{ij}\mathbf g_{ij}=\mathbf G(\mathbf Q).
    \label{eq:ot_loss_gradient_link}
\end{equation}

{
\subsection*{Proof of Theorem~\ref{thm:ot_transfer}}
\begin{proof}
Let $C(\mathbf Q)=\sum_{i,j}Q_{ij}c_{ij}$, $a_i=b_j=1/K$, and $Q^{\dagger}_{ij}=K^{-1}\mathbf 1\{j=\pi(i)\}$. The feasible plan $\mathbf Q^{\dagger}$ and Assumption~\ref{ass:1} give
\begin{equation}
    C(\mathbf Q_{ot})-C(\mathbf Q^{\dagger})
    =\sum_{i,j}Q_{ot,ij}(c_{ij}-c_{i\pi(i)})
    \geq\Delta q(\mathbf Q_{ot}).
    \label{eq:ot_cost_mass_proof}
\end{equation}
Optimality of $\mathbf Q_{ot}$ gives
\begin{equation}
    C(\mathbf Q_{ot})-C(\mathbf Q^{\dagger})
    \leq\varepsilon_{ot}\big[H(\mathbf Q_{ot})-H(\mathbf Q^{\dagger})\big]
    \leq\varepsilon_{ot}\log K,
    \label{eq:ot_entropy_proof}
\end{equation}
where $H(\mathbf Q^{\dagger})=\log K$ and $H(\mathbf Q_{ot})\leq2\log K$. Hence Eq.~\textup{(\ref{eq:ot_mismatch_bound})} holds.

Let $E(\mathbf Q)=\sum_{i,j}Q_{ij}e_{ij}$, so $|E(\mathbf Q)|\leq\xi$ for unit-mass plans. Assumption~\ref{ass:2} yields
\begin{equation}
    \langle\mathbf g_v,\mathbf G(\mathbf Q)\rangle
    =A\big[1-q(\mathbf Q)\big]-Bq(\mathbf Q)+E(\mathbf Q)
    =A-(A+B)q(\mathbf Q)+E(\mathbf Q).
    \label{eq:ot_gradient_decomposition}
\end{equation}
Subtracting the two plan contributions gives
\begin{equation}
\begin{aligned}
    \langle\mathbf g_v,\mathbf G(\mathbf Q_{ot})-\mathbf G(\mathbf Q_0)\rangle
    &=(A+B)\big[q_0-q(\mathbf Q_{ot})\big]+E(\mathbf Q_{ot})-E(\mathbf Q_0)\\
    &\geq(A+B)(q_0-\bar q)-2\xi=\kappa.
\end{aligned}
    \label{eq:ot_mass_gradient_identity}
\end{equation}
The condition in Eq.~\textup{(\ref{eq:ot_joint_transfer_condition})} gives $\bar q<q_0-2\xi/(A+B)$, hence $\kappa>0$.

Since $\mathbf h$ cancels, $\langle\mathbf g_v,\mathbf d_{ot}-\mathbf d_0\rangle\geq\lambda_l\kappa$. For $r\in\{ot,0\}$, $L$-smoothness gives
\begin{equation}
    \left|\mathcal J_v(\boldsymbol\delta+\alpha\mathbf d_r)
    -\mathcal J_v(\boldsymbol\delta)-\alpha\langle\mathbf g_v,\mathbf d_r\rangle\right|
    \leq\frac{L\alpha^2}{2}\|\mathbf d_r\|_2^2.
    \label{eq:ot_taylor_bound}
\end{equation}
Subtracting the uniform-step upper bound from the OT-step lower bound yields
\begin{equation}
\begin{aligned}
    \mathcal J_v(\boldsymbol\delta+\alpha\mathbf d_{ot})
    -\mathcal J_v(\boldsymbol\delta+\alpha\mathbf d_0)
    &\geq\alpha\lambda_l\kappa
    -\frac{L\alpha^2}{2}\big(\|\mathbf d_{ot}\|_2^2+\|\mathbf d_0\|_2^2\big)\\
    &\geq\alpha\lambda_l\kappa-L\alpha^2D^2=\tau.
\end{aligned}
    \label{eq:ot_finite_step_proof}
\end{equation}
The step-size condition ensures $\tau>0$ for $L>0$; for $L=0$, $\tau=\alpha\lambda_l\kappa>0$.
\end{proof}
}

\subsection*{Finite Solver Accuracy}
For a balanced numerical plan with exact prescribed marginals, suppose its regularized objective is at most that of $\mathbf Q^{\dagger}$ plus $\zeta\geq0$. The same argument gives
\begin{equation}
    q(\mathbf Q)\leq\min\!\left\{1,
       \frac{\varepsilon_{ot}\log K+\zeta}{\Delta}\right\}.
    \label{eq:ot_inexact_bound}
\end{equation}
For this inexact plan, the corresponding sufficient condition becomes $(\varepsilon_{ot}\log K+\zeta)/\Delta+2\xi/(A+B)<q_0$. Thus, the optimization error must also fit within the available comparison margin. The scaling-vector stopping threshold in Section~\ref{sec:adaptive_uot} is not itself a bound on $\zeta$.

\subsection*{Relation to the Implemented Update}
For a raw plan $\mathbf P$ with total mass $s=\sum_{i,j}P_{ij}>0$, let $\mathbf Q=\mathbf P/s$. The implementation uses $\widetilde{\mathbf P}=\mathbf P/(s+\eta)=\theta\mathbf Q$, where $\theta=s/(s+\eta)>0$. Thus, with the local-loss weight $\lambda_l$, its local ascent contribution is $\lambda_l\theta\mathbf G(\mathbf Q)$, corresponding to an effective local weight $\lambda_l\theta$ for a fixed plan. If two compared plans have different factors $\theta_1$ and $\theta_0$, this difference must be retained. For $\|\mathbf g_{ij}\|_2\leq G_{\max}$, and a unit-mass directional advantage bounded below by $\kappa$, one has
\begin{equation}
    \langle\mathbf g_v,\theta_1\mathbf G(\mathbf Q_1)-\theta_0\mathbf G(\mathbf Q_0)\rangle
    \geq\theta_1\kappa-|\theta_1-\theta_0|\,\|\mathbf g_v\|_2G_{\max}.
    \label{eq:ot_normalization_correction}
\end{equation}
This follows by adding and subtracting $\theta_1\mathbf G(\mathbf Q_0)$. For directions $\mathbf d_r=\mathbf h+\lambda_l\theta_r\mathbf G(\mathbf Q_r)$, the finite-step comparison uses $\lambda_l$ times this corrected first-order bound and bounds on the resulting directions. Holding the remaining branch contributions and ensemble weights fixed isolates the effect of the local matching rule. The same decomposition can be summed over encoders and views with fixed nonnegative weights; it does not compare complete trajectories with changing weights or cluster assignments.

Algorithm~\ref{alg:iaufoa} uses projected sign updates. Let $\mathcal D$ denote the feasible perturbation box, including the image-value constraints, and define the actual displacement
\begin{equation}
    \mathbf s_r=\Pi_{\mathcal D}\!\left(\boldsymbol{\delta}
      +\alpha\operatorname{sign}(\mathbf d_r)\right)-\boldsymbol{\delta}.
\end{equation}
If $\mathcal J_v$ is $L$-smooth along these displacement segments, Taylor's inequality gives
\begin{equation}
    \mathcal J_v(\boldsymbol{\delta}+\mathbf s_1)
    -\mathcal J_v(\boldsymbol{\delta}+\mathbf s_0)
    \geq\langle\mathbf g_v,\mathbf s_1-\mathbf s_0\rangle
    -\frac L2\big(\|\mathbf s_1\|_2^2+\|\mathbf s_0\|_2^2\big).
    \label{eq:ot_projected_comparison}
\end{equation}
To connect this displacement comparison to the preceding direction bound, let $\nu$ be a lower bound on $\langle\mathbf g_v,\mathbf d_1-\mathbf d_0\rangle$. For the ordinary directions of Theorem~\ref{thm:ot_transfer}, $\nu=\lambda_l\kappa$; with the implemented normalization, use $\nu=\lambda_l[\theta_1\kappa-|\theta_1-\theta_0|\,\|\mathbf g_v\|_2G_{\max}]$ from Eq.~\eqref{eq:ot_normalization_correction} instead. Define $\mathbf b_r:=\mathbf s_r-\alpha\mathbf d_r$, the change in displacement caused by taking signs and projecting. Then
\begin{equation}
\begin{aligned}
    \langle\mathbf g_v,\mathbf s_1-\mathbf s_0\rangle
    &=\alpha\langle\mathbf g_v,\mathbf d_1-\mathbf d_0\rangle
      +\langle\mathbf g_v,\mathbf b_1-\mathbf b_0\rangle\\
    &\geq\alpha\nu-\|\mathbf g_v\|_2\|\mathbf b_1-\mathbf b_0\|_2.
\end{aligned}
    \label{eq:ot_displacement_direction_link}
\end{equation}
Combining Eqs.~\eqref{eq:ot_projected_comparison} and~\eqref{eq:ot_displacement_direction_link} therefore gives
\begin{equation}
\begin{aligned}
    &\mathcal J_v(\boldsymbol\delta+\mathbf s_1)
    -\mathcal J_v(\boldsymbol\delta+\mathbf s_0)\\
    &\quad\geq\alpha\nu
    -\|\mathbf g_v\|_2\|\mathbf b_1-\mathbf b_0\|_2
    -\frac L2\bigl(\|\mathbf s_1\|_2^2+\|\mathbf s_0\|_2^2\bigr).
\end{aligned}
    \label{eq:ot_projected_transfer_bound}
\end{equation}
Thus, a positive raw-direction margin transfers to a strict advantage for the implemented update whenever the displacement-change term and the curvature term together are smaller than $\alpha\nu$. This is an additional sufficient condition, not a consequence of the transport mass bound. In particular, making $\alpha$ small does not automatically remove the effect of taking signs, since that displacement change can also be of order $\alpha$. The equations identify the extra quantities needed to pass from ordinary gradients to projected sign steps, without assuming that the ordering is preserved. Theorem~\ref{thm:ot_transfer} explains a local mechanism under explicit assumptions; it does not prove unconditional improvement of the full attack or discrete ASR. Whether these assumptions hold for a particular closed-source victim is not established by the proof. The reported experiments evaluate the complete method.

\subsection*{A Constructive Example of Possible Improvement}
The sufficient conditions are nonempty, and improvement is possible even with sign updates. Consider a scalar perturbation $t\in[-0.01,0.01]$, two surrogate clusters, and target features $\mathbf e_1,\mathbf e_2\in\mathbb R^3$. Define unit surrogate features
\begin{equation}
\begin{aligned}
    \mathbf z_1(t)&=(a(t),b(t),\sqrt{1-a(t)^2-b(t)^2}),\\
    \mathbf z_2(t)&=(b(t),a(t),\sqrt{1-a(t)^2-b(t)^2}),
    \qquad a(t)=0.9+t,\quad b(t)=0.1-2t.
\end{aligned}
\end{equation}
These features are smooth throughout the feasible interval. Their cosine similarities to the targets are $S_{11}=S_{22}=a(t)$ and $S_{12}=S_{21}=b(t)$. Let the victim feature be $\mathbf z_v(t)=(t,\sqrt{1-t^2})$ with target $(1,0)$, giving the distinct victim alignment objective $\mathcal J_v(t)=t$. At $t=0$, the identity correspondence has cost gap $\Delta=0.8$; its pairwise derivatives are $1$, whereas the off-diagonal derivatives are $-2$. Thus, Eq.~\textup{(\ref{eq:ot_gradient_separation})} holds with $A=1$, $B=2$, and $\xi=0$.

For balanced entropic OT with $\varepsilon_{ot}=0.1$, symmetry and the uniform marginals give the exact off-diagonal mass
\begin{equation}
    q(\mathbf Q_{ot})=\frac{1}{1+\exp(8)},\qquad
    G(\mathbf Q_{ot})=1-3q(\mathbf Q_{ot})>0,
    \qquad G(\mathbf Q_0)=-\frac12.
\end{equation}
With no other branch, the same positive local weight, and step size $\alpha=0.01$, the projected sign steps from $t=0$ yield $t_{ot}=0.01$ and $t_0=-0.01$. Hence $\mathcal J_v(t_{ot})>\mathcal J_v(0)>\mathcal J_v(t_0)$. This explicit cosine-feature construction establishes that OT can improve a victim objective relative to uniform matching, including under a sign update. It is an existence example, not a model of the evaluated MLLMs or evidence that their gradients satisfy the assumptions.

\clearpage

\section{Visualization of Adversarial Examples}
\label{sec:appendix_visualization}
Figure~\ref{fig:visualization} compares the adversarial images and perturbations
generated by COA, AdvDiffVLM, AnyAttack, M-Attack-v2, MPCAttack, and IAU-FOA
on three source--target pairs. The left panel shows the clean sources and
adversarial images, while the right panel displays the targets and corresponding
perturbations. IAU-FOA preserves the main structures and colors of the source
images and achieves the highest PSNR and SSIM in all three displayed examples.
From top to bottom, its PSNR values are $29.8$, $30.3$, and $32.5$ dB, with
corresponding SSIM values of $0.808$, $0.770$, and $0.729$.
The average PSNR and SSIM are $30.9$ dB and $0.769$, respectively, compared
with $30.1$ dB and $0.702$ for COA, the strongest baseline on both average
metrics. These comparisons illustrate the visual fidelity of IAU-FOA on the
displayed examples; aggregate image quality over the full evaluation set is
not assessed here.

\begin{figure}[!htbp]
\begin{center}
 \includegraphics[width=1\linewidth]{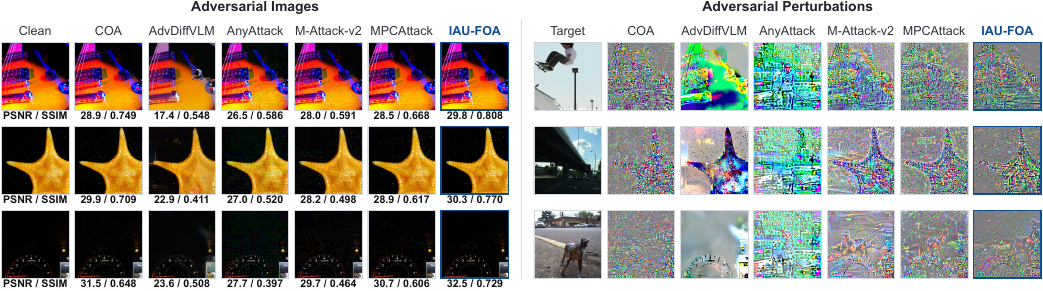}
\end{center}
\caption{\textbf{Adversarial examples and perturbations for three source--target pairs.}
Each row shows one source--target pair. Left: clean source images and adversarial
images generated by COA, AdvDiffVLM, AnyAttack, M-Attack-v2, MPCAttack, and IAU-FOA.
Right: target images and the corresponding perturbation visualizations.
PSNR (dB) and SSIM below each adversarial image are measured against its clean
source; higher values indicate better visual fidelity.}
\label{fig:visualization}
\end{figure}

\clearpage
\section{Effect of the Number of Attack Iterations}
\label{sec:appendix_iterations}
All results in Fig.~\ref{fig:iteration} are evaluated on 100 source--target image pairs per target model.
Figure~\ref{fig:iteration} evaluates the effect of the optimization budget by
varying the number of attack iterations over $\{100,200,300,400,500\}$
while keeping all other settings fixed. Increasing the budget from 100 to 300 iterations improves ASR
on all six commercial MLLMs and raises the mean ASR from $69.5\%$ to $85.5\%$,
a gain of $16.0$ percentage points. Among the tested budgets, 300 iterations
achieve the highest mean ASR. Increasing the budget to 400 iterations lowers
ASR on every target and reduces the mean to $82.5\%$. At 500 iterations, the
mean ASR is $84.5\%$, still $1.0$ percentage point below the result at 300
iterations. Relative to 300 iterations, 500 iterations improve ASR on G-3.5-F
and GPT-5.4, leave it unchanged on Cl-S-4.6, and reduce it on the other three
targets. Thus, additional iterations beyond 300 do not consistently improve
targeted transferability. Considering the highest observed mean ASR and the
lower computational cost relative to larger budgets, we retain 300 iterations
as the default setting unless otherwise specified.

\begin{figure}[!htbp]
\centering
\begin{minipage}{\dimexpr(\textwidth-\columnsep)/2\relax}
\centering
\includegraphics[width=\linewidth]{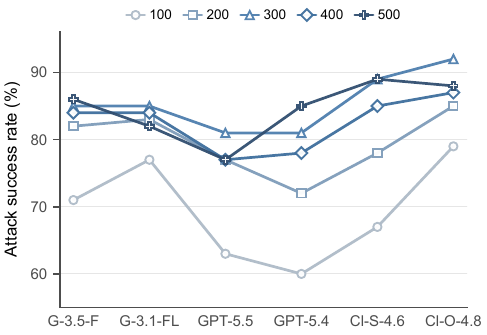}
\caption{\textbf{Effect of attack iterations on targeted transferability.}
ASR on six commercial MLLMs for 100--500 iterations, evaluated on 100 source--target image pairs per model. The highest mean ASR
($85.5\%$) occurs at 300 iterations.}
\label{fig:iteration}
\end{minipage}
\end{figure}

\clearpage

\section{Optimization Algorithm}
\label{sec:optimization_algorithm}

Algorithm~\ref{alg:iaufoa} summarizes the alignment and perturbation updates for supplied source transformations and target views. Spatial view preparation is represented by the input view schedule; visual-invariance augmentation, adaptive transport, and dynamic weighting are shown explicitly. We use $\eta=10^{-8}$ and $\tau_w=1.0$; transport plans, confidence modulation, and ensemble weights are held fixed during backpropagation.

\begin{algorithm}[h]
\caption{IAU-FOA Optimization}
\label{alg:iaufoa}
\begin{algorithmic}[1]
\REQUIRE source image $\mathbf{x}_s$, target image $\mathbf{x}_{tar}$, surrogate encoder ensemble $\mathcal{F}=\{f_m\}_{m=1}^{M}$, perturbation budget $\epsilon$, iterations $T$, step size $\alpha$, local-loss weight $\lambda_l$, marginal KL weight $\rho$, fixed number of local clusters $K$ (default: $10$), prepared view schedule $\{\mathcal V^{(t)}\}_{t=0}^{T-1}$.
\ENSURE adversarial image $\mathbf{x}_{adv}$.
\STATE Initialize $\boldsymbol{\delta}^{(0)}=\mathbf{0}$, $\mathbf{x}_{adv}^{(0)}=\mathbf{x}_s+\boldsymbol{\delta}^{(0)}$, and evaluation index $k=1$.
\STATE Set $\mathbf a_m^{adv}=\mathbf a_m^t=\mathbf 1/K$, $\varepsilon_{ot}=0.1$, and $\phi=\rho/(\rho+\varepsilon_{ot})$.
\FOR{$t=0$ to $T-1$}
    \STATE Obtain prepared view pairs $(\mathcal T,\mathbf x_{tar}^{c})\in\mathcal V^{(t)}$, with VIA parameters $(s,\boldsymbol\beta)$ in $\mathcal T$.
    \STATE Set accumulated gradient $\mathbf{g}=\mathbf{0}$ and valid-copy count $n=0$.
    \FOR{each $(\mathcal T,\mathbf x_{tar}^{c})\in\mathcal V^{(t)}$}
        \STATE $\mathbf{x}^{c}=\mathcal{T}(\mathbf{x}_{adv}^{(t)})$; obtain global target features and $\mathbf Y_m^{clu}=\operatorname{KMeans}(\mathbf Z_m^p(\mathbf x_{tar}^{c}),K)$ without gradient tracking.
        \FOR{$m=1$ to $M$}
            \STATE $\mathcal{L}_{g}^{m}=1-\cos\left(\hat{\mathbf{z}}_{m}^{g}(\mathbf{x}^{c}),\hat{\mathbf{z}}_{m}^{g}(\mathbf{x}_{tar}^{c})\right)$.
            \STATE $\mathbf{X}_{m}^{clu}=\operatorname{KMeans}(\mathbf{Z}_{m}^{p}(\mathbf{x}^{c}),K)$.
            \STATE $\mathbf{S}_{m}(i,j)=\cos\left(\mathbf{x}_{m,i}^{clu},\mathbf{y}_{m,j}^{clu}\right)$.
            \STATE $\mathbf{K}_{m}(i,j)=\exp\left(-(1-\mathbf{S}_{m}(i,j))/\varepsilon_{ot}\right)$.
            \STATE $r_i=\max(\max_j\mathbf{S}_{m}(i,j),0)$, $\hat{r}_i=r_i/\big(\frac{1}{K}\sum_j r_j+\eta\big)$, $\phi_i^{a}=\rho\hat{r}_i/(\rho\hat{r}_i+\varepsilon_{ot})$ with $\gamma=1$.
            \STATE Initialize $\mathbf u=\mathbf v=\mathbf 1$; update $\mathbf v$ then $\mathbf u$ by Eq.~\textup{(\ref{eq:sinkhorn_update})}, for at most 100 rounds or until the mean change in $\mathbf v$ is below $10^{-2}$.
            \STATE $\mathbf{P}_{m}^{*}=\mathrm{diag}(\mathbf{u})\mathbf{K}_{m}\mathrm{diag}(\mathbf{v})$.
            \STATE $\widetilde{\mathbf{P}}_m=\operatorname{stopgrad}\!\left(\mathbf{P}_m^*/(\sum_{i,j}\mathbf{P}_m^*(i,j)+\eta)\right)$.
            \STATE $\mathcal{L}_{l}^{m}=\sum_{i,j}\widetilde{\mathbf{P}}_{m}(i,j)(1-\mathbf{S}_{m}(i,j))$.
            \STATE $\mathcal{J}_{m}^{c}=\mathcal{L}_{g}^{m}+\lambda_l\mathcal{L}_{l}^{m}$; $A_m^{(k)}=1-\mathcal{L}_{g}^{m}+\lambda_l\sum_{i,j}\widetilde{\mathbf{P}}_m(i,j)\mathbf{S}_m(i,j)$.
        \ENDFOR
        \STATE If $k=1$, set $A_m^{(0)}=\operatorname{stopgrad}(A_m^{(1)})$ for all $m$.
        \STATE $q_m^{(k)}=\operatorname{stopgrad}(A_m^{(k)})/(A_m^{(k-1)}+\eta)$, $m=1,\ldots,M$.
        \STATE $\omega_m^{(k)}=M\cdot\exp(q_m^{(k)}/\tau_w)/\sum_{j=1}^{M}\exp(q_j^{(k)}/\tau_w)$, $m=1,\ldots,M$.
        \STATE $\mathcal{J}^{c}=\sum_{m=1}^{M}\omega_m^{(k)}\mathcal{J}_{m}^{c}$; store $A_m^{(k)}\leftarrow\operatorname{stopgrad}(A_m^{(k)})$ for all $m$; $k\leftarrow k+1$.
        \STATE $\mathbf{g}\leftarrow \mathbf{g}+\nabla_{\boldsymbol{\delta}}\mathcal{J}^{c}$; $n\leftarrow n+1$.
    \ENDFOR
    \STATE $\mathbf{g}\leftarrow\mathbf{g}/\max(n,1)$.
    \STATE $\boldsymbol{\delta}^{(t+1)}
    =
    \Pi_{[-\epsilon,\epsilon]}
    \left(
    \boldsymbol{\delta}^{(t)}
    -
    \alpha\cdot \mathrm{sign}(\mathbf{g})
    \right)$.
    \STATE $\mathbf{x}_{adv}^{(t+1)}=\mathbf{x}_s+\boldsymbol{\delta}^{(t+1)}$.
\ENDFOR
\STATE \textbf{return} $\mathbf{x}_{adv}^{(T)}$.
\end{algorithmic}
\end{algorithm}

\clearpage
\section{Source--Target Similarity by Attack Outcome}
\label{sec:source_target_similarity}

To investigate whether the semantic proximity of a source--target image pair is
associated with targeted attack success, we partition the IAU-FOA adversarial examples
for each target MLLM into successful and failed subsets and compute the mean
CLIP similarity between the corresponding clean source and target images. As
reported in Table~\ref{tab:clip_similarity}, successful attacks exhibit higher
source--target similarity than failed attacks on all 12 target models.
The table also reports the corresponding ASR. Within each model category,
the Avg column aggregates successful and failed group means separately,
weighted by their respective sample counts, and averages ASR across models.
The differences in similarity are small and are interpreted as descriptive
associations between source--target similarity and attack outcome.
Because the analysis groups samples by attack outcome after evaluation, it
does not establish that increasing source--target similarity causes higher ASR.

\begin{table}[htbp]
\centering
\scriptsize
\setlength{\tabcolsep}{4pt}
\renewcommand{\arraystretch}{1.02}
\caption{%
  \textbf{Source--target CLIP similarity by attack outcome.}
  Mean similarity for successful and failed IAU-FOA attacks, with the corresponding ASR.
  Avg reports sample-count-weighted means for the successful and failed groups,
  and the arithmetic mean of ASR across models. Values are reported to three decimal places.
  Q, I, M, L, G, and Cl
  denote Qwen2.5-VL, InternVL3, MiniCPM-o, LLaVA-1.6, Gemini, and Claude;
  F, FL, S, and O denote Flash, Flash-Lite, Sonnet, and Opus, respectively.}
\label{tab:clip_similarity}
\begin{tabular}{@{}l*{7}{c}@{}}
\toprule
\textbf{Open-source}
& Q-7B
& Q-72B
& I-14B
& I-78B
& M-2.6
& L-7B
& \textbf{Avg} \\
\midrule
\textbf{Successful}
& 0.476 & 0.475 & 0.474 & 0.474 & 0.474 & 0.474 & \textbf{0.474} \\
\textbf{Failed}
& 0.462 & 0.465 & 0.467 & 0.469 & 0.469 & 0.469 & \textbf{0.466} \\
\textbf{ASR}
& 0.821 & 0.823 & 0.910 & 0.871 & 0.866 & 0.885 & \textbf{0.863} \\

\bottomrule
\addlinespace[2pt]
\textbf{Closed-source}
& G-3.5-F
& G-3.1-FL
& GPT-5.5
& GPT-5.4
& Cl-S-4.6
& Cl-O-4.8
& \textbf{Avg} \\
\midrule
\textbf{Successful}
& 0.474 & 0.475 & 0.478 & 0.475 & 0.476 & 0.474 & \textbf{0.475} \\
\textbf{Failed}
& 0.469 & 0.465 & 0.463 & 0.469 & 0.464 & 0.470 & \textbf{0.466} \\
\textbf{ASR}
& 0.776 & 0.789 & 0.696 & 0.672 & 0.744 & 0.822 & \textbf{0.750} \\
\bottomrule
\end{tabular}
\end{table}

\clearpage
\section{Evaluation Prompt}
\label{sec:prompt}

For reproducibility, Fig.~\ref{fig:prompt} presents the complete prompt used by
GPT-4o to compute GPTScore. For each adversarial--target pair, the two
placeholders are replaced with the descriptions generated by the victim MLLM
for the adversarial image and the corresponding target image. The judge is
instructed to focus on subject-level semantic consistency, tolerate minor
differences in wording and detail, and return only a scalar score in $[0,1]$.
We use this scalar directly when computing AvgSim and classify an attack as
successful when the score is at least $0.5$. The same prompt and decision
rule are applied to all attack methods and victim models.

\begin{figure}[h]
    \centering
    \includegraphics[width=0.8\linewidth]{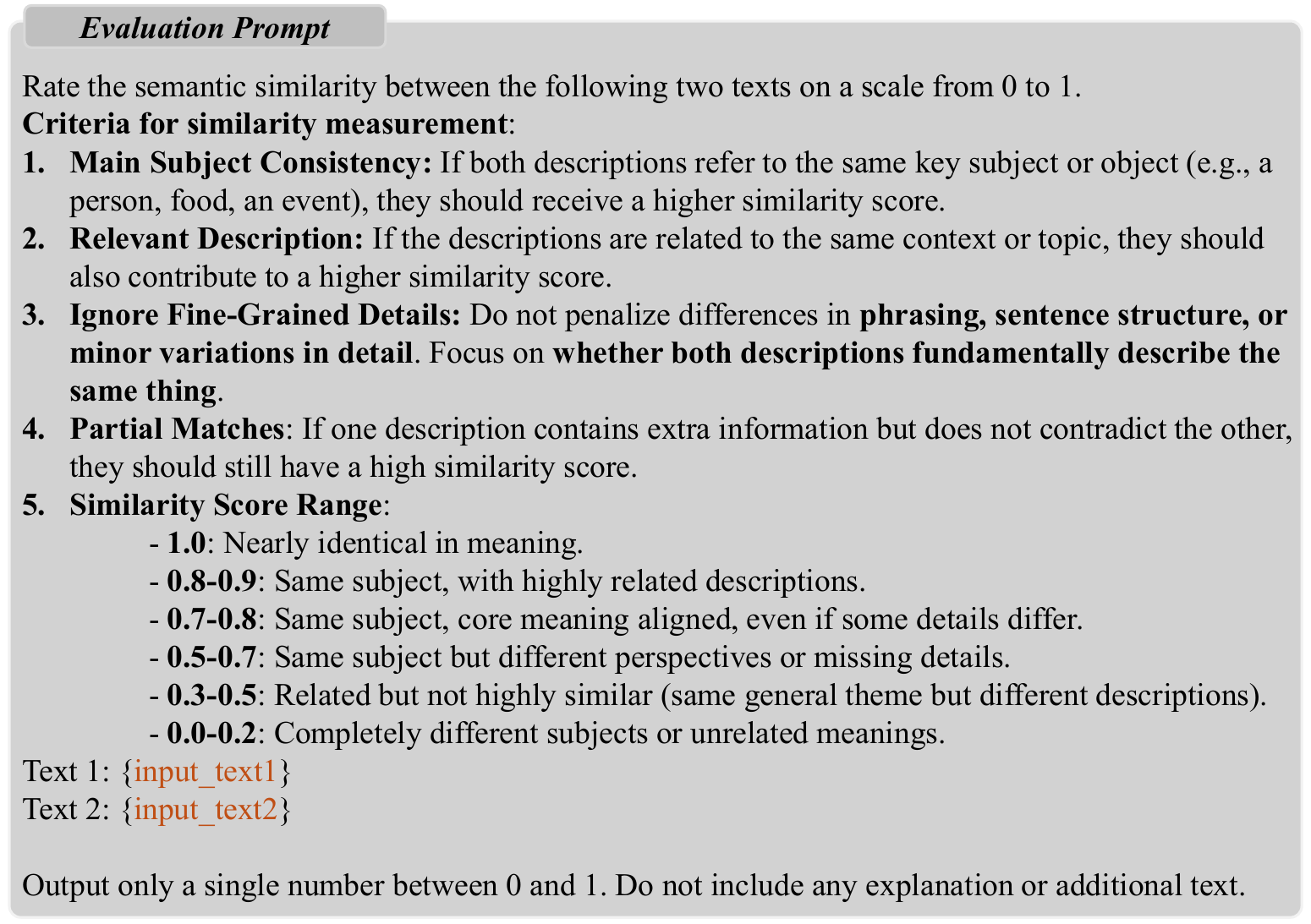}
    \caption{GPT-4o evaluation prompt for GPTScore. Placeholders contain the victim MLLM's descriptions of the adversarial and target images.}
    \label{fig:prompt}
\end{figure}

\end{document}